\documentclass[twoside,11pt]{article}
\usepackage{jmlr2e}

\usepackage{booktabs}
\usepackage{amsfonts}       %
\usepackage{nicefrac}       %
\usepackage{microtype}      %
\usepackage{amsmath}
\usepackage{amssymb}
\usepackage{algorithm}
\usepackage{algpseudocode}
\usepackage{flushend}
\usepackage{mathrsfs}
\usepackage{tikz}
\usepackage{pgfplots}
\usepackage{pgf}
\usepackage{bm}
\usepackage{graphicx}
\usepackage{lastpage}
\usepackage{epstopdf,enumitem}
\usepackage{tikz-cd}
\DeclareGraphicsExtensions{.eps}

\usepgflibrary{shapes}
\usetikzlibrary{%
	arrows,%
	decorations.text,%
	positioning,%
	scopes,%
	shapes%
}
\usepackage{subfigure}

\newtheorem{theorem}{Theorem}
\newtheorem{lemma}{Lemma}
\newtheorem{assumption}{Assumption}
\newtheorem{proposition}{Proposition}

\newtheorem{definition}{Definition}

\begin{document}
\ShortHeadings{Reinforcement Learning for Joint Optimization of Multiple Rewards}{Agarwal, Aggarwal}

\title{Reinforcement Learning for Joint Optimization of Multiple Rewards}  
\author{Mridul Agarwal and Vaneet Aggarwal\\Purdue University, West Lafayette IN 47907} 
	\editor{}
\maketitle

\begin{abstract}  %
Finding optimal policies which maximize long term rewards of Markov Decision Processes requires the use of dynamic programming and backward induction to solve the Bellman optimality equation. However, many real-world problems require optimization of an objective that is non-linear in cumulative rewards for which dynamic programming cannot be applied directly.
For example, in a resource allocation problem, one of the objectives is to maximize long-term fairness among the users. We notice that when an agent aim to optimize some function of the sum of rewards is considered, the problem loses its Markov nature. This paper addresses and formalizes the problem of optimizing a non-linear function of the long term average of rewards. We propose model-based and model-free algorithms to learn the policy, where the model-based policy is shown to achieve a regret of $\Tilde{O}\left(LKDS\sqrt{\frac{A}{T}}\right)$ for $K$ objectives combined with a concave $L$-Lipschitz function.  Further, using the fairness in cellular base-station scheduling, and queueing system scheduling as examples, the proposed algorithm is shown to significantly outperform the conventional RL approaches.
\end{abstract}

\if 0
\fi

%
%
\section{Introduction}
Many practical applications of sequential decision making often have multiple objectives. For example, a hydro-power project may have conflicting gains with respect to power generation and flood management \citep{castelletti2013multiobjective}. Similarly, a building climate controller can have conflicting objectives of saving energy and maximizing comfort of the residents of the building \citep{kwak2012saves}. Video streaming applications also account for multiple objectives like stall duration and average video quality \citep{elgabli2018lbp}. Access of files from cloud storage aims to optimize the latency of file download and the cost to store the files \citep{xiang2015joint}. Many applications also require to allocate resources fairly to multiple clients \citep{lan2010axiomatic} which can be modelled as optimizing a function of the objectives of the individual clients. This paper aims to provide a novel formulation for decision making among multiple objectives using reinforcement learning approaches and to analyze the performance of the proposed algorithms.

We consider a setup where we want to optimize a possibly nonlinear joint objective function of long-term rewards  of all the objectives (or, different objectives). As an example, many problems in resource allocation for networking and computation resources use fairness among the long-term average rewards of the users as the metric of choice \citep{lan2010axiomatic,kwan2009proportional,bu2006generalized,li2018resource,aggarwal2011characterizing,margolies2016exploiting,wang2014dominant,ibrahim2010leen}, which is a non-linear metric. For fairness optimization, a controller wants to optimize a fairness objective among the different agents, e.g., proportional fairness, $\alpha$-fairness, or improve the worst-case average reward of the users \citep{altman2008generalized}. %
In such situations, the overall joint objective function cannot be written as sum utility at each time instant. This prohibits the application of standard single-agent reinforcement learning based policies as the backward induction step update cannot be directly applied here.  For example, if a process has $2$ agents and $T>1$ steps, and all the resource was allocated to the first agent till $T-1$ steps. Then, at $T^{th}$ step the resource should be allocated to the second agent to ensure fairness. This requires the need to track past allocation of all the resources and not just the current state of the system. We also note that the optimal policy cannot take a deterministic action in a state in general, and thus the optimal policy is not a deterministic policy in general. Consider a case where a scheduler needs to fairly allocate a resource between two users, and the system has only one state. A deterministic policy policy will allocate the resource to only one of the user, and hence is not optimal.
We define a novel multi-agent formulation, making several practical assumptions, which optimizes the joint function of the average per-step rewards of the different objectives to alleviate the need for maintaining history. 

SARSA and Q-Learning algorithms \citep{sutton2018reinforcement},  and their deep neural network based DQN algorithm \citep{mnih2015human}  provide  {\color{black}policies that depend only on the current state}, hence are sub-optimal. Further, these algorithms learn a Q-value function {\color{black}which can be computed based on a dynamic programming approach, which is not valid in our work.} Using evaluations on fair resource allocation and network routing problems, we verify that algorithms based on finding fixed point of Bellman equations do not perform well. {\color{black} This further motivates the need for novel RL based algorithms to optimize non-linear functions.}

We further note that even though multi-agent reinforcement learning algorithms have been widely studied, \citep{tan1993multi,shoham2003multi,bucsoniu2010multi,ono1996multi}, there are no convergence proofs to the optimal joint objective function without the knowledge of the transition probability, to the best of our knowledge. This paper assumes no knowledge of the state transition probability of the objectives and aims to provide algorithms for the decision making of the different objectives. We provide two algorithms; The first is a model-based algorithm that learns the transition probability of the next state given the current state and action. The second algorithm is model-free, which uses policy gradients to find the optimal policy.

The proposed model-based algorithm uses posterior sampling with Dirichlet distribution. We show that the proposed algorithm converges to an optimal point when the joint objective function is Lipschitz continuous. Since the optimal policy is a stochastic policy, policy search space is not-finite. We show that the problem is convex under a certain class of functions and can be efficiently solved. In setups where the joint objective function is max-min, the setup reduces to a linear programming optimization problem. In addition, we show that the proposed algorithm achieves a regret bound sub-linear in the number of time-steps and number of objectives. To obtain the regret bound, we use a Bellman error based analysis to analyze stochastic policies. The Bellman error quantifies the difference in rewards for deviating from the true MDP for one step and then following the true MDP thereafter. Using this analysis, our regret bound characterizes the gap between the optimal objective and the objective achieved by the algorithm in $T$ time-steps. We show a regret bound of {\color{black} $\Tilde{O}\left(KDS\sqrt{\frac{A}{T}}\right)$}, where $K, T$ denotes the number of objectives,  and time steps, respectively.

The  proposed model-free algorithm can be easily implemented using deep neural networks for any differentiable objective function. Further, we note that the reward functions of the different objectives can be very different, and can optimize different metrics for the objectives. As long as there is a joint objective function, the different objectives can make decisions to optimize this function and achieve the optimal decision at convergence.

The proposed framework works for any number of objectives, while is novel even for a single agent ($K=1$). In this case, the agent wishes to optimize a non-linear concave function of the average reward. Since this function is not assumed to be monotone, optimizing the function is not equivalent to optimizing the average reward. For any general non-linear concave function, regret bound is analyzed for model-based case. 

We also present evaluation results for both the algorithms for optimizing proportional fairness of multiple agents connecting to a cellular base station. We compare the obtained policies with existing asymptotically optimal algorithm {\color{black}(Blind Gradient Estimator or BGE)} of optimizing proportional fairness for wireless networks \citep{margolies2016exploiting} and SARSA based RL solution proposed by \citep{perez2009responsive}. We developed a simulation environment for wireless network for multiple number of agents and states for each agent. The proposed algorithm significantly outperform the SARSA based solution, and {\color{black} it nearly achieves the performance of the asymptotically optimal BGE algorithm.} We also considered $\alpha$-fairness for an infinite state space to show the scalability of the proposed model-free algorithm. In this case, the domain-specific algorithm was not available, while we show that we outperform Deep Q-Network (DQN) based algorithm \citep{mnih2015human}. {\color{black}Finally}, a queueing system is considered which models multiple roads merging into a single lane. The queue selection problem is modeled using the proposed framework and the proposed approach is shown to improve the fair latency reward metric among the queues significantly as compared to the DQN and the longest-queue-first policies.

Key contributions of our paper are:
\begin{itemize}[leftmargin=*]
    \item A structure for joint function optimization with multiple objectives based on average per-step rewards.
    \item {\color{black} Pareto Optimality guarantees when the joint objective is an element-wise monotone function. }
    \item A model-based algorithm using posterior sampling with Dirichlet distribution, and its regret bounds.
    \item A model-free policy gradient algorithm which can be efficiently implemented using neural networks.
    \item Evaluation results and comparison with existing heuristics for optimizing fairness in cellular networks, and queueing systems.
\end{itemize}

The rest of the paper is organized as follows. Section \ref{related_work} describes related works in the field of RL and MARL. Section \ref{model} describes the problem formulation. Pareto optimality of the proposed framework is shown in Section \ref{pareto_optimal}.  The proposed  model based algorithm and model free algorithm are described in Sections  \ref{mba} and \ref{mfa}, respectively. In  Section \ref{simulations}, the proposed algorithms are evaluated for cellular scheduling problem. Section \ref{conclusion} concludes the paper with some future work directions. %
\section{Related Work} \label{related_work}
Reinforcement learning for single agent has been extensively studied in past \citep{sutton2018reinforcement}. Dynamic Programming was used in many problems by finding cost to go at each stage \citep{puterman1994markov,bertsekas1995dynamic}. These models optimize linear additive utility and utilize the power of Backward Induction.

Following the success of Deep Q Networks \citep{mnih2015human}, many new algorithms have been developed for reinforcement learning  \citep{schulman2015trust,lillicrap2015continuous,wang2015dueling,schulman2017proximal}. These papers focus on single agent control, and provide a framework for implementing scalable algorithms. Sample efficient algorithms based on rate of convergence analysis have also been studied for model based RL algorithms \citep{agrawal2017optimistic,osband2013more}, and for model free Q learning \citep{jin2018q}. However, sample efficient algorithms use tabular implementation instead of a deep learning based implementation.

Owing to high variance in the policies obtained by standard Markov Decision Processes and Reinforcement Learning formulations, various authors worked in reducing the \textit{risk} in RL approaches \citep{garcia2015comprehensive}. Even though the risk function (e.g., Conditional Value at Risk (CVaR)) is non-linear in the rewards, this function is not only a function of long-term average rewards of the single agent {\color{black} but also a function of the higher order moments of the rewards of the single agent}. Thus, the proposed framework does not apply to the risk measures. However, for both the risk measure and general non-linear concave function of average rewards, optimal policy is non-Markovian.

Non-Markovian Decision Processes is a class of decision processes where either rewards, the next state transitions, or both do not only depends on the current state and actions but also on the history of states and actions leading towards the current state. One can augment the state space to include the history along with the current state and make the new process Markovian \citep{thiebaux2006decision}. However, this increases the memory footprint of any Q-learning algorithm. \citep{mccallum1995instance} considers only $H$ states of history to construct an approximate MDP and then use Q-learning. \citep{li2006towards} provide guarantees on Q-learning for non-MDPs where an agent observes and work according to an abstract MDP instead of the ground MDP. The states of the abstract MDP are an abstraction of the states of the ground MDP. \citep{hutter2014extreme} extend this setup to work with abstractions of histories. \citep{majeed2018q} consider a setup for History-based Decision Process (HDP). They provide convergence guarantees for Q-learning algorithm for a sub-class of HDP where for histories $h$ and $h'$, $Q(h, a) = Q(h', a)$ if the last observed state is identical for both $h$ and $h'$. They call this sub-class $Q$-value uniform Decision Process (QDP) and this subsumes the abstract MDPs. We note that our work is different from these as the $Q$-values constructed using joint objective is not independent of history.

In most applications such as financial markets, swarm robotics, wireless channel access, etc., there are multiple agents that make a decision \citep{bloembergen2015evolutionary}, and the decision of any agent {\color{black}can possibly} affect the other agents. In early work on multi-agent reinforcement learning (MARL) for stochastic games \citep{littman1994markov}, it was recognized that no agent works in a vacuum. In his seminal paper, Littman \citep{littman1994markov} focused on only two agents that had opposite and opposing goals. This means that they could use a single reward function which one tried to maximize and the other tried to minimize. The agent had to work with a competing agent and had to behave to maximize their reward in the worst possible case. In MARL, the agents select actions simultaneously at the current state and receive rewards at the next state. Different from the {\color{black}frameworks} that solve for a Nash equilibrium in a stochastic game, the goal of a reinforcement learning algorithm is to learn equilibrium strategies through interaction with the environment \citep{tan1993multi,shoham2003multi,bucsoniu2010multi,ono1996multi,shalev2016safe}.

\citep{survey_MOMDP,liu2014multiobjective,nguyen2020multi} considers the multi-objective Markov Decision Processes. Similar to our work, they consider function of expected cumulative rewards. However, they work with linear combination of the cumulative rewards whereas we consider a possibly non-linear function $f$. Further, based on the joint objective as a function of expected average rewards, we provide regret guarantees for our algorithm. For joint decision making, \citep{zhang2014fairness,zhang2015fairness} studied the problem of fairness with multiple agents and related the fairness to multi-objective MDP. They considered maximin fairness and used Linear Programming to obtain optimal policies. For general functions, linear programming based approach provided by \citep{zhang2014fairness} will not directly work. This paper also optimizes joint action of agents using a centralized controller and propose a model based algorithm to obtain optimal policies. Based on our assumptions, maximin fairness becomes a special case of our formulation and optimal policies can be obtained using the proposed model based algorithm. We also propose a model free reinforcement learning algorithm that can be used to obtain optimal policies for any general differentiable functions of average per-step rewards of individual agents. Recently, \citep{jiang2019fen} considered the problem of maximizing fairness among multiple agents. However, they do not provide a convergence analysis for their algorithm. {\color{black} We attempt to close this gap in the understanding of the problem of maximizing a concave and Lipschitz function of multiple objectives with our work.}
\section{Problem Formulation} \label{model}

We consider an infinite horizon discounted {Markov decision process (MDP) $\mathcal{M}$ defined by the tuple $\left(\mathcal{S}, \mathcal{A}, P, r^1, r^2, \cdots, r^K, \rho_0\right)$}. $\mathcal{S}$ denotes a finite set of state space of size $S$, and $\mathcal{A}$ denotes a finite set of $A$ actions. $P:\mathcal{S}\times\mathcal{A}\to[0,1]^S$ denotes the probability transition distribution. $K$ denotes the number of objectives and $[K] = \{1, 2, \cdots, K\}$ is the set of $K$ objectives. {\color{black}Let $r^k:\mathcal{S}\times\mathcal{A}\to[0,1]$ be the bounded reward function for objective $k\in[K]$}. Lastly, $\rho_0:\mathcal{S}\to[0,1]$ is the distribution of initial state. We motivate our choice of bounds on rewards from the fact that many problems in practice require explicit reward shaping. Hence, the controller or the learner is aware of the bounds on the rewards. We consider the bounds to be $[0,1]$ for our case which is easy to satisfy by reward shaping.

We use a stochastic policy $\pi : \mathcal{S} \times \mathcal{A}\to [0,1]$ which returns the probability of selecting action $a \in \mathcal{A}$ for any given state $s \in \mathcal{S}$. Following policy $\pi$ on the MDP, the agent observes a sequence of random variables $\{S_t, A_t\}_t$ where $S_t$ denotes the state of the agent at time $t$ and $A_t$ denotes the action taken by the agent at time $t$. The expected discounted long term reward and expected per step reward of the objective $k$ are given by $J_{\pi}^{P,k}$ and $\lambda_{\pi}^{P,k}$, respectively, when the joint policy $\pi$ is followed. Formally, for discount factor $\gamma\in(0,1)$, $J_{\pi}^{P,k}$ and $\lambda_{\pi}^{P,k}$ are defined as
\begin{align}
    J_{\pi}^{P,k} &= \mathbb{E}_{S_0, A_0, S_1, A_1,\cdots}\left[\lim_{\tau\to\infty}\sum_{t=0}^{\tau}\gamma^t r^k(S_t,A_t)\right]\\
    S_0&\sim \rho_0(S_0),\ A_t\sim \pi(A_t|S_t),\ S_{t+1}\sim P(\cdot|S_t, A_t)\nonumber\\
    \lambda_{\pi}^{P,k} &= \mathbb{E}_{S_0, A_0, S_1, A_1,\cdots}\left[\lim_{\tau\to\infty}\frac{1}{\tau}\sum_{t=0}^{\tau} r^k\left(S_t,A_t\right)\right]\\
    &= \lim_{\gamma\to1}(1-\gamma)J_{\pi}^{P,k}\label{eq:lambda_from_J}
\end{align}
Equation \eqref{eq:lambda_from_J} follows from the Laurent series expansion of $J_\pi^k$ \citep{puterman1994markov}. For brevity, in the rest of the paper $\mathbb{E}_{S_t, A_t, S_{t+1}; t\geq 0}[\cdot]$ will be denoted as $\mathbb{E}_{\rho, \pi, P}[\cdot]$, where $S_0\sim \rho_0,\ A_t\sim \pi(\cdot|S_t),\ S_{t+1}\sim P(\cdot|S_t, A_t)$. {\color{black}The expected per step reward satisfies the following Bellman equation
\begin{align}
    h_{\pi}^{P,k}(s) + \lambda_{\pi}^{P,k} = \mathbb{E}_{a\sim\pi}\left[r^k(s, a)\right] + \mathbb{E}_{a\sim\pi}\left[\sum_{s'}P(s'|s,a)h_{\pi}^{P,k}(s')\right]
\end{align}
where $h_\pi^k(s)$ is the bias of policy $\pi$ for state $s$. We also define the discounted value function $V_{\gamma,\pi}^{P, k}(s)$ and Q-value functions $Q_{\gamma,\pi}^{P, k}(s,a)$ as follows:
\begin{align}
    V_{\gamma, \pi}^{P, k}(s) &= \mathbb{E}_{A_0, S_1, A_1,\cdots}\left[\lim_{\tau\to\infty}\sum_{t=0}^{\tau}\gamma^t r^k(S_t,A_t)|S_0 = s\right]\\
    Q_{\gamma, \pi}^{P, k}(s,a) &= \mathbb{E}_{S_1, A_1,\cdots}\left[\lim_{\tau\to\infty}\sum_{t=0}^{\tau}\gamma^t r^k(S_t,A_t)|S_0 = s, A_0= a\right]\nonumber\\
    &= r^k(s, a) + \mathbb{E}_{a\sim\pi}[\sum_{s'}P(s'|s,a)V_{\gamma, \pi}^{P,k}(s')]
\end{align}
Further, the bias $h_{\pi}^{P,k}(s)$ and the value function $V_{\gamma,\pi}^{P,k}$ are related as
\begin{align}
h_{\pi}^{P,k}(s_1) - h_{\pi}^{P,k}(s_2) = \lim_{\gamma \to 1}\left(V_{\gamma, \pi}^{P,k}(s_1) - V_{\gamma,\pi}^{P,k}(s_2)\right)\text{, where } s_1, s_2\in \mathcal{S}.
\end{align}
For notation simplicity, we may drop the superscript $P$ when discussing about variables for the true MDP.
}

Note that each policy induces a Markov Chain on the states $\mathcal{S}$ with transition probabilities $P_{\pi, s}(s') = \sum_{a\in\mathcal{A}}\pi(a|s)P(s'|s,a)$. After defining a policy, we can now define the diameter of the MDP $\mathcal{M}$ as:
\begin{definition}[Diameter]
Consider the Markov Chain induced by the policy $\pi$ on the MDP $\mathcal{M}$. Let $T(s'|\mathcal{M}, \pi, s)$ be a random variable that denotes the first time step when this Markov Chain enters state $s'$ starting from state $s$. Then, the diameter of the MDP $\mathcal{M}$ is defined as:
\begin{align}
    D(\mathcal{M}) = \max_{\pi}\max_{s'\neq s}\mathbb{E}\left[T(s'|\mathcal{M}, \pi, s)\right]
\end{align}
\end{definition}

Further, starting from an arbitrary initial state distribution, the state distribution may take a while to converge to the steady state distribution. For any policy $\pi$, let $P_{\pi,s}^t = \left(P_{\pi, s}\right)^t$ be the $t$-step probability distribution of the states when policy $\pi$ is applied to MDP with transition probabilities $P$ starting from state $s$.

We are now ready to state our first assumption on the MDP. We assume that the Markov Decision Process is ergodic. This implies that: \textbf{1.} for any  policy all states, $s\in\mathcal{S}$, communicate with each other;  \textbf{2.} for any policy, the process converges to the steady state distribution exponentially fast. Formally, we have

\begin{assumption}\label{bounded_mixing_time_assumtion}
The Markov Decision Process, $\mathcal{M}$, is ergodic. Then, we have, 

1. The diameter, $D$, of $\mathcal{M}$ is finite. 

2. For any policy $\pi$, for some $C> 0$ and $0\le\rho<1$, we have
\begin{align}
    \|P_{\pi,s}^t - d_{\pi}\|_{TV} \le C\rho^t
\end{align}
where $d_\pi$ is the steady state distribution induced by policy $\pi$ on the MDP.
\end{assumption}

\color{black}

{The agent aim to collaboratively optimize the function $f:\mathbb{R}^K\to\mathbb{R}$, which is defined over the long-term rewards of the individual objectives.} We make certain practical assumptions on this joint objective function $f$, which are listed as follows:

\begin{assumption}\label{concave_assumption}
The objective function $f$ is jointly concave. Hence for any arbitrary distribution $\mathcal{D}$, the following holds.
\begin{align}
    f\left(\mathbb{E}_{\mathbf{x}\sim\mathcal{D}}\left[\mathbf{x}\right]\right) \geq \mathbb{E}_{\mathbf{x}\sim\mathcal{D}}\left[f\left(\mathbf{x}\right)\right];\ \mathbf{x}\in\mathbb{R}^K\label{eq:concave_utility}
\end{align}
\end{assumption}
The objective function $f$ represents the utility obtained from the expected per step reward for each objective. These utility functions are often concave to reduce variance for a risk averse decision maker \citep{10.2307/1913738}. To model this concave utility function, we assume the above form of Jensen's inequality. 

\begin{assumption}\label{lipschitz_assumption}
The function $f$ is assumed to be a $L-$ Lipschitz function, or
    \begin{align}
        &\left|f\left(\mathbf{x}\right) - f\left(\mathbf{y}\right)\right| \leq L\left\lVert \mathbf{x} - \mathbf{y}\right\rVert_1;\ \mathbf{x}, \mathbf{y}\in\mathbb{R}^K \label{eq:Lipschitz}
    \end{align}
\end{assumption}
Assumption \ref{lipschitz_assumption} ensures that for a small change in long run rewards for any objective does not cause unbounded changes in the objective.

Based on Assumption \ref{concave_assumption}, we maximize the function of expected sum of rewards for each objective. Further to keep the formulation independent of time horizon or $\gamma$, we maximize the function over expected per-step rewards of each objective. Hence, our goal is to find the optimal policy as the solution for the following optimization problem.
\begin{align}
\pi^* = &\arg\max_{\pi}f(\lambda^1_{\pi}, \cdots, \lambda^K_{\pi}) \label{eq:optimal_policy_def}
\end{align}
If $f(\cdot)$ is also monotone, we note that the optimal policy in \eqref{eq:optimal_policy_def} can be shown to be Pareto optimal. The detailed proof will be presented later in Section \ref{pareto_optimal}.

Any online algorithm $\mathscr{A}$ starting with no prior knowledge will require to obtain estimates of transition probabilities $P$ and obtain rewards $r_k, \forall\ k\in[K]$ for each state action pair. Initially, when algorithm $\mathscr{A}$ does not have good estimates of the model, it accumulates a regret for not working as per optimal policy.  We define a time dependent regret $R_T$ to achieve an optimal solution defined as the difference between the optimal value of the function and the value of the function at time $T$, or
{
\begin{align}
    R_T& = \mathbb{E}_{S_t, A_t}\Bigg[\Big|f\left(\lambda^1_{\pi^*}, \cdots, \lambda^K_{\pi^*}\right) - f\left(\frac{1}{T}\sum_{t=0}^Tr^1(S_t, A_t), \cdots, \frac{1}{T}\sum_{t=0}^Tr^K(S_t, A_t)\right)\Big|\Bigg] \label{eq:regret_defined}
\end{align}
The regret defined in Equation \eqref{eq:regret_defined} is the expected deviation between the value of the function obtained from the expected rewards of the optimal policy and the value of the function obtained from the observed rewards from a trajectory. Following the work of \citep{survey_MOMDP}, we note that the outer expectation comes for running the decision process for a different set of users or running a separate and independent instance for the same set of users. Since the realization can be different from the expected rewards, the function values can still be different even when following the optimal policy.}%

We note that we do not require $f(\cdot)$ to be monotone. Thus, even for a single objective, optimizing {\color{black}$f(\mathbb{E}[\sum_{t}r_t])$} is not equivalent to optimizing $\mathbb{E}[\sum_{t}r_t]$. Hence, the proposed framework can be used {to} optimize functions of cumulative or long term average reward for single objective as well.

In the following sections, we first show that the joint-objective function of average rewards allows us to obtain Pareto-optimal policies with an additional assumption of monotonicity. Then, we will present a model-based algorithm to obtain this policy $\pi^*$, and regret accumulated by the algorithm.  We will present a model-free algorithm in Section  \ref{mfa} which can be efficiently implemented using Deep Neural Networks.
\section{Obtaining Pareto-Optimal Policies}\label{pareto_optimal}

Many multi-objective or multi-agent formulations require the policies to be Pareto-Optimal \citep{survey_MOMDP,sener2018multi,van2014multi}. The conflicting rewards of various agents may not allow us to attain simultaneous optimal average rewards for any agent with any joint policy. Hence, an envy-free Pareto optimal policy is desired.
We now provide an additional assumption on the joint objective function, and show that the optimal policy satisfying Equation \eqref{eq:optimal_policy_def} is Pareto optimal.

\begin{assumption}
If $f$ is an element-wise monotonically strictly increasing function. Or, $\forall\ k\in[K]$, the function satisfies,
    \begin{align}
        x^k > y^k \implies f\left(x^1,\cdots,x^{k},\cdots,x^K\right) > f\left(x^1,\cdots,y^{k},\cdots,x^K\right) \label{eq:monotone_prop}
    \end{align}
\end{assumption}

Element wise increasing property motivates the agents to be strategic as by increasing its per-step average reward, agent can increase the joint objective. Based on Equation \eqref{eq:monotone_prop}, we notice that the solution for Equation \eqref{eq:optimal_policy_def} is Pareto optimal.
\begin{definition}
A policy $\pi^*$ is said to be \textbf{Pareto optimal} if and only if there is exists no other policy $\pi$ such that the average per-step reward is at least as high for all agents, and strictly higher for at least one agent. In other words, 
\begin{align}
    \forall\ k\in[K],\ \lambda_{\pi^*}^k \geq \lambda_{\pi}^k\ and\ \exists\  k,\ \lambda_{\pi^*}^k > \lambda_{\pi}^k \label{eq:Pareto_optimal}
\end{align}
\end{definition}

\begin{theorem}
Solution of Equation \eqref{eq:optimal_policy_def}, or the optimal policy $\pi^*$ is Pareto Optimal.
\end{theorem}
\begin{proof}
We will prove the result using contradiction. Let $\pi^*$ be the solution of Equation \eqref{eq:optimal_policy_def} and not be Pareto optimal. Then there exists some policy $\pi$ for which the following equation holds,
\begin{align}
    \forall\ k\in[K],\ \lambda_{\pi}^k \geq \lambda_{\pi^*}^k\ and\ \exists\  k,\ \lambda_{\pi}^k > \lambda_{\pi^*}^k \label{eq:non_pareto_optimal}
\end{align}
From element-wise monotone increasing property in Equation \eqref{eq:monotone_prop}, we obtain
\begin{align}
    f(\lambda_{\pi^*}^1,\cdots, \lambda_{\pi}^k, \cdots,\lambda_{\pi^*}^K) &> f(\lambda_{\pi^*}^1,\cdots, \lambda_{\pi^*}^k, \cdots,\lambda_{\pi^*}^K)\\
            &= \arg\max_{\pi'}f(\lambda^1_{\pi'}, \cdots, \lambda^K_{\pi'})
\end{align}
This is a contradiction. Hence, $\pi^*$ is a Pareto optimal solution.
\end{proof}

This result shows that algorithms presented in this paper can be used to optimally allocate resources among multiple agents using average per step allocations.
\section{Model-based Algorithm} \label{mba}

RL problems typically optimize the cumulative rewards, which is a linear function of rewards at each time step because of the addition operation. This allows the Bellman Optimality Equation to require only the knowledge of the current state to select the best action to optimize future rewards \citep{puterman1994markov}. 
However, since our controller is optimizing a joint non-linear function of the long-term rewards from multiple sources, Bellman Optimality Equations cannot be written as a function of the current state exclusively. Our goal is to find the optimal policy as solution of Equation \eqref{eq:optimal_policy_def}. Using average per-step reward and infinite horizon allows us to use Markov policies. An intuition into why this works is there is always infinite time available to optimize the joint objective $f$.
\color{black}

The individual long-term {average}-reward for each agent is still linearly additive ($\frac{1}{\tau}\sum_{t=0}^\tau r^k(S_t, A_t)$). %
For infinite horizon optimization problems (or $\tau\to\infty$), we can use steady state distribution of the state to obtain expected cumulative rewards. For all $k\in[K]$, we use
    \begin{align}
        \lambda_{\pi}^k = \sum_{s\in \mathcal{S}}\sum_{a\in \mathcal{A}}r^k(s, a)d_{\pi}(s,a) \label{eq:average_reward_from_ss_dist}
    \end{align}
where $d_{\pi}(s,a)$ is the steady state joint distribution of the state and actions under policy $\pi$. Equation \eqref{eq:average_reward_from_ss_dist} suggests that we can transform the optimization problem in terms of optimal policy to optimal steady-state distribution. Thus, we have the joint optimization problem in the following form which uses steady state distributions
    \begin{align}
        d^* = \arg\max_{d} f\left(\sum_{s\in \mathcal{S}, a\in \mathcal{A}}r^1(s, a)d(s,a), \cdots, \sum_{s\in \mathcal{S}, a\in \mathcal{A}}r^K(s, a)d(s,a)\right) \label{eq:optimization_equation}
    \end{align}
with the following set of constraints,
\begin{align}
    \sum_{a\in\mathcal{A}}d(s',a) &= \sum_{s\in\mathcal{S}, a\in\mathcal{A}}P(s'|s, a)d(s, a)\ \forall\ s'\in\mathcal{S}\label{eq:transition_constraint}\\
    \sum_{s\in\mathcal{S}, a\in\mathcal{A}} d(s, a) &= 1\label{eq:total_prob_constraint}\\
    d(s, a) &\geq 0 \ \forall\ s\in\mathcal{S}, a\in\mathcal{A}\label{eq:non_neg_prob_constrant}
\end{align}
\color{black}
Constraint (\ref{eq:transition_constraint}) denotes the transition structure for the underlying Markov Process. Constraint (\ref{eq:total_prob_constraint}), and constraint (\ref{eq:non_neg_prob_constrant}) ensures that the solution is a valid probability distribution. Since $f(\cdots)$ is jointly concave, arguments in Equation (\ref{eq:optimization_equation}) are linear, and the constraints in Equation (\ref{eq:transition_constraint},\ref{eq:total_prob_constraint},\ref{eq:non_neg_prob_constrant}) are linear, this is a convex optimization problem. Since convex optimization problems can be solved in polynomial time \citep{bubeck2015convex}, we can use standard approaches to solve Equation (\ref{eq:optimization_equation}). After solving the optimization problem, we find the optimal policy from the obtained steady state distribution $d^*(s,a)$ as,
\begin{align}
    \pi^*(a|s) = \frac{Pr(a, s)}{Pr(s)} = \frac{d^*(a, s)}{\sum_{a\in\mathcal{A}}d^*(s, a)}\label{eq:optimal_policy}
\end{align}

The proposed model-based algorithm estimates the transition probabilities by interacting with the environment. We need the steady state distribution $d_\pi$ to exist for any policy $\pi$. { We note that when the priors of the transition probabilities $P(\cdot|s,a)$ are a Dirichlet distribution for each state and action pair, such a steady state distribution exists}. Proposition \ref{steady_state_proposition} formalizes the result of the existence of a steady state distribution when the transition probability is sampled from a Dirichlet distribution.
\begin{proposition}\label{steady_state_proposition}
For MDP $\mathcal{M}$ with state space $\mathcal{S}$ and action space $\mathcal{A}$, let the transition probabilities ${P}$ come from a Dirichlet distribution. Then, any policy $\pi$ for $\mathcal{M}$ will have a steady state distribution $\hat{d}_\pi$ given as
    \begin{align}
        \hat{d}_\pi(s') = \sum_{s\in\mathcal{S}}\hat{d}_\pi(s)\left(\sum_{a\in\widehat{\mathcal{A}}} \pi(a|s)P(s'|s, a)\right) \forall s'\in\mathcal{S}.
    \end{align}
\end{proposition}
\begin{proof}
The transition probabilities $P(s'|s, a)$ follow Dirichlet distribution, and hence they are strictly positive. Further, as the policy $\pi(a|s)$ is a probability distribution on actions conditioned on state, $\pi(a|s) \geq 0,\ \sum_a \pi(a|s) = 1$. So, there is a non zero transition probability to reach from state $s\in{\mathcal{S}}$ to state $s'\in{\mathcal{S}}$.

Now, note that all the entries of the transition probability matrix are strictly positive. And, hence the Markov Chain induced over the MDP $\mathcal{M}$ by any policy $\pi$ is 1) irreducible, as it is possible to reach any state from any other state, and 2) aperiodic, as it is possible to reach any state in a single time step from any other state. Together, we get the existence of the steady-state distribution \citep{lawler2018introduction}.
\end{proof}

To complete the setup for our algorithm, we make few more assumptions stated below.

\begin{assumption}\label{dirichlet_prior_assumption}
The transition probabilities $P(\cdot|s,a)$ of the Markov Decision Process have a Dirichlet prior for all state action pairs $(s,a)$.
\end{assumption}

Since we assume that transition probabilities of the MDP $\mathcal{M}$ follow Dirichlet distributions, all policies on $\mathcal{M}$ have a steady-state distribution.

\begin{algorithm}[thbp]
	\caption{Model-Based Joint Decision Making Algorithm} \label{alg:model_based_algo}
    \begin{algorithmic}[1]
        \Procedure{Model Based Online}{$\mathcal{S}, \mathcal{A}, [K], f, r$}
        \State Initialize $N(s, a, s') = 1~~\forall~~(s, a, s') \in \mathcal{S}\times\mathcal{A}\times\mathcal{S}$
        \State Initialize $\nu_0(s, a) = 0~~\forall~~(s, a) \in \mathcal{S}\times\mathcal{A}$
        \State Initialize $\pi_1(a|s) = \frac{1}{|\mathcal{A}|}\ \forall\ (a, s)\in\mathcal{A}\times\mathcal{S}$
        \State Initialize $e = 1$
        
        \For{time steps $t = 1, 2, \cdots$}
            \State Observe current state $S_t$
            \State Sample action to play $A_t \sim \pi_e(\cdot|S_t)$
            \State Play $A_t$, obtain reward $r_t\in[0,1]^K$ and observe next state $S_{t+1}$
            \State Update $N(S_t, A_t, S_{t+1}) \leftarrow N(S_t, A_t, S_{t+1})+  1$, $\nu_e(S_t, A_t) \leftarrow \nu_e(S_t, A_t)+ 1$
            \If{$\nu_e(S_t, A_t) \geq \max\{1, \sum_{e'< e}\nu_{e'}(S_t, A_t)\}$}
                \State Update epoch counter $e \leftarrow e + 1$
                \State Initialize $\nu_e(s, a) = 0~~\forall~~(s, a) \in \mathcal{S}\times\mathcal{A}$
    	        \State $P_e(s'|s, a) \sim Dir(N(s, a,\cdot))\ \forall\ (s, a)$
    	        \State Solve steady state distribution $d(s, a)$ as the  solution of the optimization problem in  Equations (\ref{eq:optimization_equation}-\ref{eq:non_neg_prob_constrant})
    	        \State Obtain optimal policy $\pi_e$ as 
    	        $$\pi_e(a|s) = \frac{d(s, a)}{\sum_{a\in\mathcal{A}}d(s, a)}$$                
            \EndIf
        \EndFor
        \EndProcedure
    \end{algorithmic}
\end{algorithm}

\subsection{Algorithm Description}
Algorithm \ref{alg:model_based_algo} describes the overall procedure that estimates the transition probabilities and the reward functions. The algorithm takes as input the state space $\mathcal{S}$, action space $\mathcal{A}$, set of agents $[K]$, the reward structure $r$, and the objective function $f$. It initializes the next state visit count for each state-action pair $N(s, a, s')$ by one for Dirichlet sampling. For initial exploration, the policy uses a uniform distribution over all the actions. The algorithm proceeds in epochs $e$ with $t_e$ denoting the time step $t$ at which epoch $e$ starts. Also, $\nu_e(s,a)$ stores the number of times a state-action pair is visited in epoch $e$. We assume that the controller is optimizing for infinite horizon and thus there is no stopping condition for epoch loop in Line 6. For each time index $t\in[t_e, t_{e+1})$ in epoch $e$, the controller observes the state, samples and plays the action according to $\pi_e(\cdot|s)$, and observes the rewards for each agent and next state. It then updates the state visit count for the observed state and played action pair. We break the current epoch $e$ if the total number of visitations of any state action pair in current epoch exceeds total visitations in previous epoch. This epoch breaking strategy ensures that each state-action triggers epoch switches exponentially slower. This increasing epoch duration bounds the policy switching cost \citep{jaksch2010near}. After breaking the epoch, we sample transition probabilities from the updated posterior and find a new optimal policy for the sampled MDP by solving the optimization framework described in Equation(\ref{eq:optimization_equation}).

\subsection{Regret} \label{regret}
We now prove the regret bounds of Algorithm \ref{alg:model_based_algo} in the form of the following theorem. We first give the high level ideas used in obtaining the bounds on regret. The Lipschitz property of the objective function $f$ allows to write the regret as sum of regrets for individual per-step average reward $\lambda_\pi^k$. Now, we can divide the regret into regret incurred in each epoch $e$. Then, we use the posterior sampling lemma (Lemma 1 from \citep{osband2013more}) to obtain the equivalence between the value of the function $f$ for the optimal policy of the true MDP $\mathcal{M}$ and the value of the function for the optimal value of the sampled MDP $\widehat{\mathcal{M}}$ conditioned on the observed state evolution. Now, for any epoch $e$, note that for each reward $k$, the agent plays policy $\pi_e$ optimized for the sampled MDP $\widehat{\mathcal{M}}$ on the true MDP $\mathcal{M}$. We break the difference between the two terms, the per-step average reward $\lambda_{\pi_e}^{P_e,k}$ of the policy $\pi_e$ for the sampled MDP with transition probability $P_e$ and the rewards obtained into two terms. The first term is the difference $\lambda_{\pi_e}^{P_e,k}$ and $\lambda_{\pi_e}^{P,k}$. We will denote our immediate discussion on bounding this term. The second deviation is the difference between $\lambda_{\pi_e}^{P,k}$ and the observed rewards $r^k_t$ for $t\in[t_e, t_{e+1})$, which we can bound using Azuma-Hoeffding's inequality.

To compute the regret incurred by the optimal policy $\pi_e$ for the sampled MDP on the true MDP $\mathcal{M}$, we use Bellman error. For some policy $\pi$, we define Bellman error $B_\pi^{\Tilde{P},k}(s,a)$ for the infinite horizon MDPs as the difference between the cumulative expected rewards obtained for deviating from the system model with transition $\Tilde{P}$ for one step by taking action $a$ in state $s$ and then following policy $\pi$. We have:
\begin{align}
    B_{\gamma,\pi}^{\Tilde{P},k}(s,a) &= \Big(Q_{\gamma, \pi}^{\Tilde{P},k}(s,a) - r(s,a) -  \gamma\sum\nolimits_{s'\in\mathcal{S}}P(s'|s,a)V_{\gamma, \pi}^{\Tilde{P}, k}(s,a)\Big)\nonumber\\
    B_\pi^{\Tilde{P},k}(s,a)&= \lim_{\gamma\to 1}B_{\gamma,\pi}^{\Tilde{P},k}(s,a) \label{eq:Bellman_error_definition}
\end{align}

We relate the Bellman error defined in Equation \eqref{eq:Bellman_error_definition} to the gap between the expected per step reward $\lambda_{\pi}^{\Tilde{P},k}$ for running policy $\pi$ on an MDP with transition probability $\Tilde{P}$ and the expected per step reward $\lambda_{\pi}^{P,k}$ for running policy $\pi$ on the true MDP in the following lemma:

\begin{lemma}\label{lem:bound_average_by_bellman}
The difference of long-term average rewards for running the policy $\pi_e$ on the MDP, $\lambda_{\pi}^{\Tilde{P},k}$, and the average long-term average rewards for running the policy $\pi$ on the true MDP, $\lambda_{\pi}^{\Tilde{P},k}$, is the long-term average Bellman error as
\begin{align}
    \lambda_{\pi}^{\Tilde{P},k} - \lambda_{\pi}^{P,k} = \sum_{s,a}d_{\pi}(s,a) B_{\pi}^{\Tilde{P},k}(s,a) = \mathbb{E}_{(s,a)\sim\pi, P}\left[B_{\pi}^{\Tilde{P},k}(s,a)\right]
\end{align}
where $d_{\pi}$ is the occupancy measure generated by policy $\pi$ on the true MDP and $\mathbb{E}_{(s,a)\sim\pi, P}[\cdot]$ denotes expectation over $s,a$ sampled when running policy $\pi$ on MDP with transition probability $P$.
\end{lemma}
\begin{proof}[Proof Sketch]
We start by writing $Q_{\gamma, \pi}^{\Tilde{P}, k}$ in terms of the Bellman error. Now, subtracting $V_{\gamma,\pi}^{P, k}$ from $V_{\gamma,\pi}^{\Tilde{P},k}$ and using the fact that $\lambda_{\pi_e}^{P,k} = \lim_{\gamma\to1}(1-\gamma)V_{\gamma,\pi}^{P,k}$ and $\lambda_{\pi_e}^{\Tilde{P},k} = \lim_{\gamma\to1}(1-\gamma)V_{\gamma,\pi_e}^{\Tilde{P},k}$ we obtain the required result.
A complete proof is provided in Appendix \ref{app:average_error_is_bellman_error}.
\end{proof}

After relating the gap between the long-term average rewards of policy $\pi_e$ on the two MDPs, we now want to bound the sum of Bellman error over an epoch. For this, we first bound the Bellman error for a particular state action pair $s,a$ in the form of following lemma. We have,
\begin{lemma}\label{lem:bound_bellman_s_a_main}
For an MDP with rewards $r^k(s,a)$ and transition probability $\Tilde{P}(s'|s,a)$ such that $\|\Tilde{P}(\cdot|s,a) - P(\cdot|s,a)\|_1 \le \epsilon_{s,a}$, the Bellman error $B_{\pi_e}^{\Tilde{P},k}(s,a)$ for state-action pair $s,a$ is upper bounded as
\begin{align}
B_{\pi}^{\Tilde{P},k}(s,a) \le \Big\|\Tilde{P}(\cdot|s,a) - P(\cdot|s,a)\Big\|_1\|h_{\pi}^{\Tilde{P},k}(\cdot)\|_\infty\le\Tilde{D}\min\left\{\epsilon_{s,a}, 2\right\},
\end{align}
where $\Tilde{D}$ is the diameter of the MDP with transition probability $\Tilde{P}$.
\end{lemma}
\begin{proof}[Proof Sketch]
We start by noting that the Bellman error essentially bounds the impact of the difference in value obtained because of the difference in transition probability to the immediate next state.
A complete proof is provided in Appendix \ref{app:bounding_bellman_error}.
\end{proof}

A major part of the regret analysis is how well the learned model estimates the true system model. For that, bound the deviation of the estimates of the estimated transition probabilities of the Markov Decision Processes $\mathcal{M}$. For that we use $\ell_1$ deviation bounds from \citep{weissman2003inequalities}. Consider, the following event,
\begin{align}
    \mathcal{E}_t = \left\{\|\Bar{P}_t(\cdot|s, a) - P(\cdot|s, a)\|_1 \leq \sqrt{\frac{14S\log(2AT)}{\max\{1, n_t(s,a)\}}}\forall (s,a)\in\mathcal{S}\times\mathcal{A}\right\}
\end{align}
where $n_t(s,a)=\sum_{t' = 1}^t {\bf 1}_{\{S_{t'} = s, a_{t'}= a\}}$ and $\Bar{P}(s'|s,a) = \left(\sum_{t' = 1}^t {\bf 1}_{\{S_{t'} = s, a_{t'}= a, S_{t' + 1} = s'\}}\right)/n_t(s,a)$ of the empirical estimate of the transition probabilities. Then we have, the following lemma:

\begin{lemma}\label{lem:deviation_of_probability_estimates}
The probability that the event $\mathcal{E}_t$ fails to occur us upper bounded by $\frac{1}{20t^6}$.
\end{lemma}

We also want to ensure that the behavior of the sampled MDP from the posterior distribution is identical to the behavior of the true MDP conditioned on the observed state, action evolution. For this, we use the following lemma.

\begin{lemma}[Posterior Sampling Lemma 1 \citep{osband2013more}]\label{lem:posteriorsampling}
For any $\sigma(H_t)$-measurable function $g$, if $P$ follows distribution $\phi$, then for transition probabilities $P_e$ sampled from $\phi(\cdot|H_t)$ we have,
\begin{align}
    \mathbb{E}\left[g(P)|\sigma(H_t)\right] = \mathbb{E}\left[g(P_e)|\sigma(H_t)\right] \label{eq:equal_exp_ps}
\end{align}
\end{lemma}

Lastly, we have the following lemma to bound the number of policy switches to ensure that after every policy switch, the stochastic process of the state evolution gets sufficient steps to converge towards stationary distribution.
\begin{lemma}\label{lem:bound_on_episodes}[\citep{jaksch2010near}[Proposition 18]]
The total number of epochs $E$ for the Algorithm \ref{alg:model_based_algo} with $\nu_e(s,a) \le \sum_{e'}^{e-1}\max\{1, \nu_{e'}(s,a)\}$ for any $s,a$, up to step $T \geq SA$ is upper bounded as
\begin{align}
    E \leq 1 + 2AS + AS \log_2\left(\frac{T}{SA}\right)
\end{align}
\end{lemma}

Using Lemma \ref{lem:bound_average_by_bellman}, Lemma \ref{lem:bound_bellman_s_a_main}, Lemma \ref{lem:posteriorsampling} and Lemma  \ref{lem:deviation_of_probability_estimates}, we can now bound the regret of Algorithm \ref{alg:model_based_algo} in the form of following theorem.

\begin{theorem}
The regret $R_T$ of Algorithm \ref{alg:model_based_algo} for MDP with Dirichlet priors and diameter $D$ is bounded.
\begin{equation}
     R_T \leq \Tilde{O}\left(LKDS\sqrt{\frac{A}{T}} + \frac{1}{T}\frac{SACD}{1-\rho}\right)
\end{equation}
\end{theorem}
\begin{proof}We use the Lipschitz continuity of the function to break the scalarized objective into long-term average reward regrets of individual objecectives. Using Lipschitz continuity, the total regret becomes the sum of individual regrets.
\begin{align}
    \mathbb{E}\left[R_T\right]&=\mathbb{E}\left[\Big|f\left(\cdots, \lambda_{\pi^*}^{P,k}, \cdots\right) - f\left(\cdots, \frac{1}{T}\sum_{t=0}^Tr^k(S_t, A_t), \cdots\right)\Big|\right]\\
    &= \mathbb{E}\left[ \frac{L}{T}\sum_{k=1}^K\Big|T\lambda_{\pi^*}^{P,k} - \sum_{t=0}^Tr^k(S_t, A_t)\Big|\right] \leq  \frac{LK}{T} \max_{k\in[K]}\mathbb{E}\left[\Big|T\lambda_{\pi^*}^{P,k} - \sum_{t=0}^Tr^k(S_t, A_t)\Big|\right].
\end{align}

We can divide regret for any objective $k\in[K]$ over $T$ time steps into regret accumulated over episodes as:

\begin{align}
    \mathbb{E}\left[|T\lambda_{\pi^*}^{P,k} - \sum_{t=1}^Tr^k(S_t, A_t)|\right] &= \mathbb{E}\left[\left|\sum_{e=1}^E\sum_{t=t_e}^{t_{e+1}-1}\left(\lambda_{\pi^*}^{P,k} - \sum_{t=1}^Tr^k(S_t, A_t)\right)\right|\right]
\end{align}

Now, we note that the regret in each episode is conditioned on filtration $H_{t_e}$, and is $\sigma(H_{t_e})$-measurable. Hence, we can use \citep[{Posterior Sampling Lemma}]{osband2013more} and \cite[{Regret Equivalence Theorem}]{osband2013more} to obtain the equivalence between the per-step average reward of the optimal policy for the true MDP and the per-step average reward of the optimal policy of the sampled MDP. We have:

\begin{align}
    \mathbb{E}\left[\left|\sum_{e=1}^E\sum_{t=t_e}^{t_{e+1}-1}\left(\lambda_{\pi^*}^{P,k} - \sum_{t=1}^Tr^k(S_t, A_t)\right)\right|\right] &\le \mathbb{E}\left[\sum_{e=1}^E\left|\sum_{t=t_e}^{t_{e+1}-1}\left(\lambda_{\pi^*}^{P,k} - \sum_{t=1}^Tr^k(S_t, A_t)\right)\right|\right]\\
    &= \mathbb{E}\left[\sum_{e=1}^E\mathbb{E}\left[\left|\sum_{t=t_e}^{t_{e+1}-1}\left(\lambda_{\pi^*}^{P,k} - \sum_{t=1}^Tr^k(S_t, A_t)\right)\right|\Big|H_{t_e}\right]\right]\\
    &= \mathbb{E}\left[\sum_{e=1}^E\mathbb{E}\left[\left|\sum_{t=t_e}^{t_{e+1}-1}\left(\lambda_{\pi_e}^{P_e,k} - \sum_{t=1}^Tr^k(S_t, A_t)\right)\right|\Big|H_{t_e}\right]\right]\\
    &= \mathbb{E}\left[\sum_{e=1}^E\left|\sum_{t=t_e}^{t_{e+1}-1}\left(\lambda_{\pi_e}^{P_e,k} - \sum_{t=1}^Tr^k(S_t, A_t)\right)\right|\right]
\end{align}

We now consider two cases. The first case, (a), is where the estimated system model or the sampled system model are not close to the true system model. The second case, (b), is where both the estimated system model and the sampled system model are close the true system model. The total regret can be bounded by using the law of total expectation. Also, to reduce notational clutter, we calculate the regret conditioned on the number of episodes $E$ or $\mathbb{E}[R(T)|E]$. We can then remove the dependency on number of episodes by considering the largest possible number of episodes from Lemma \ref{lem:bound_on_episodes}.

We start by characterizing how close are the estimated transition probability and the sampled transition probability are to the true transition probability. We use $\ell_1$ distance metric for this. For all $s, a$, we construct the set of probability distributions $P'(\cdot|s, a)$,
\begin{align}
    \mathcal{P}_t = \left\{P': \|\Bar{P}_t(\cdot|s, a) - P'(\cdot|s, a)\|_1 \leq \sqrt{\frac{14S\log(2AT)}{\max\{1, n_t(s,a)\}}}\right\}\label{eq:bounded_probaiblity_estimates}
\end{align}
where $n_t=\sum_{t'}^t {\bf 1}_{\{S_{t'} = s, a_{t'}= a\}}$. Using the construction of the set $\mathcal{P}_t$, we can now define the events $\mathcal{E}_t$, and $\widehat{\mathcal{E}}_t$ as:
\begin{align}
    \mathcal{E}_t = \left\{P \in \mathcal{P}_t\right\}\text{, and } \widehat{\mathcal{E}}_t = \left\{P_e \in \mathcal{P}_t\right\}
\end{align}
Further, note that $\mathcal{P}_t$ is $\sigma(H_t), H_t = \{s_1, a_1, \cdots, s_t, a_t\}$ measurable and hence from Lemma \ref{lem:posteriorsampling} we have $\mathbb{P}(\widehat{\mathcal{E}}_t) = \mathbb{P}(P_e\in \mathcal{P}_t)= \mathbb{P}(P\in \mathcal{P}_t) = \mathbb{P}(\mathcal{E}_t)$.

We first bound the regret for the case where the system model is not well estimated or the sampled system from the posterior distribution is far from the true system model. This is equivalent to considering the event in Equation \eqref{eq:bounded_probaiblity_estimates} does not occur or the complementary events $\mathcal{E}_t^c, \hat{\mathcal{E}}_t^c$. We already bounded the probability of this event in Lemma \ref{lem:deviation_of_probability_estimates} using result from \citep{weissman2003inequalities}. In particular, we have:

\begin{align}
    \sum_{e=1}^E \mathbb{E}\left[\Big|\sum_{t=t_e}^{t_{e+1}-1}\left(\lambda_{\pi_e}^{P_e, k} -r^k(S_t, A_t)\right)\Big|\Bigg|\mathcal{E}_{t_e}^c \cup \widehat{\mathcal{E}}_{t_e}^c\right]
    &\leq \sum_{e=1}^E \mathbb{E}\left[\sum_{t=t_e}^{t_{e+1}-1}1\Bigg|\mathcal{E}_{t_e}^c \cup \widehat{\mathcal{E}}_{t_e}^c\right]\label{eq:remove_mod_from_reward_gap}\\
    &\leq \sum_{e=1}^E\sum_{s,a}\nu_k(s,a)\mathbb{P}(\mathcal{E}_{t_e}^c)\\
    &\leq  \sum_{e=1}^Et_e\mathbb{P}(\mathcal{E}_{t_e}^c\cup \widehat{\mathcal{E}}_{t_e}^c)\\
    &\leq  \sum_{t=1}^Tt\left(\mathbb{P}(\mathcal{E}_{t}^c) + \mathbb{P}( \widehat{\mathcal{E}}_{t_e}^c)\right)\label{eq:length_of_epoch}\\
    &=  2\sum_{t=1}^Tt\mathbb{P}(\mathcal{E}_{t}^c)\\
    &\leq  2\sum_{t=1}^{T^{1/4}}t\mathbb{P}(\mathcal{E}_{t}^c) + 2\sum_{t=T^{1/4}+1}^Tt\mathbb{P}(\mathcal{E}_{t}^c)\\
    &\leq  2\sum_{t=1}^{T^{1/4}}t.1 + 2\sum_{t=T^{1/4}+1}^Tt\frac{1}{t^6}\label{eq:probability_bound}\\
    &\leq  2\sqrt{T} + 2\int_{t=T^{1/4}}^\infty \frac{1}{t^5} \\
    &\leq  2\sqrt{T} + 2\frac{1}{4T} \leq  4\sqrt{T}
\end{align}
where Equation \eqref{eq:remove_mod_from_reward_gap} follows from the fact that rewards are bounded by $1$ and the modulus operator is not required on the positive sum. Equation \eqref{eq:length_of_epoch} follows from the fact that $\sum_{s,a}\nu_k(s,a)\leq \sum_{s,a}N_k(s,a) = t_k$. Further, Equation \eqref{eq:probability_bound} follows from Lemma \ref{lem:deviation_of_probability_estimates}. This completes the first case (a).

For the second case, we now break the regret into two terms as follows:
\begin{align}
    &\sum_{e=1}^E\mathbb{E}\left[\Big|\sum_{t=t_e}^{t_{e+1}-1}\left(\lambda_{\pi_e}^{P_e, k} -r^k(S_t, A_t)\right)\Big|\Bigg|\mathcal{E}_{t_e} \cap \widehat{\mathcal{E}}_{t_e}\right]\nonumber\\
    &= \sum_{e=1}^E\mathbb{E}\left[\Big| \sum_{t=t_e}^{t_{e+1}-1}\left(\lambda_{\pi_e}^{P_e, k} - \lambda_{\pi_e}^{P,k} + \lambda_{\pi_e}^{P,k}-r^k(S_t, A_t)\right)\Big|\Bigg|\mathcal{E}_{t_e} \cap \widehat{\mathcal{E}}_{t_e}\right]\\
    &\le \sum_{e=1}^E\mathbb{E}\left[\Big| \sum_{t=t_e}^{t_{e+1}-1}(\lambda_{\pi_e}^{P_e, k} - \lambda_{\pi_e}^{P,k})\Big|\Bigg|\mathcal{E}_{t_e} \cap \widehat{\mathcal{E}}_{t_e}\right] + \sum_{e=1}^E\mathbb{E}\left[\Big| \sum_{t=t_e}^{t_{e+1}-1}(\lambda_{\pi_e}^{P,k}-r^k(S_t, A_t))\Big|\Bigg|\mathcal{E}_{t_e} \cap \widehat{\mathcal{E}}_{t_e}\right]\nonumber\\
    &= R_1^k(T) + R_2^k(T).
\end{align}

The first term, $R_1^k(T)$, denotes the gap of running the optimal policy for the sampled policy on the true MDP in an epoch $e$. We bound this term with the Bellman error defined in Equation \eqref{eq:Bellman_error_definition}. The second term, $R_2^k(T)$, denotes the regret incurred from the deviation of the observed rewards and the expected per step rewards.

We now focus on the $R_1^k$ term. We begin with using Lemma \ref{lem:bound_average_by_bellman} to replace $\lambda_{\pi_e}^{P_e,k} - \lambda_{\pi_e}^{P,k}$. We have the following series of inequalities.

\begin{align}
    R_1^k(T) &= \sum_{e=1}^E\mathbb{E}\left[\Big| \sum_{t=t_e}^{t_{e+1}-1}(\lambda_{\pi_e}^{P_e, k} - \lambda_{\pi_e}^{P,k})\Big|\Bigg|\mathcal{E}_{t_e} \cap \widehat{\mathcal{E}}_{t_e}\right]\\
    &\le \sum_{e=1}^E\mathbb{E}\left[\Big| \sum_{t=t_e}^{t_{e+1}-1}(\lambda^{P_e^k}_{\pi_e} - \lambda_{\pi_e}^{P,k})\Big|\Bigg|\mathcal{E}_{t_e} \cap \widehat{\mathcal{E}}_{t_e}\right]\\
    &= \sum_{e=1}^E\mathbb{E}\left[\Big| \sum_{t=t_e}^{t_{e+1}-1}\mathbb{E}_{(s,a)\sim\pi_e, P}[B_{\pi_e}^{P_e^k, k}(s,a)]\Big|\Bigg|\mathcal{E}_{t_e} \cap \widehat{\mathcal{E}}_{t_e}\right]\label{eq:long_term_average_on_starting_distribution}
\end{align}
where $P_e^k$ is the transition probability for which $\pi_e$ maximizes $\lambda_{\pi}^{P',k}$ for $P'\in\mathcal{P}_{t_e}$. From Lemma \ref{lem:bounded_v_span_of_optimal_MDP} in Appendix, the bias-span $\|h_{\pi_e}^{P_e^k,k}\|_\infty$ is upper bounded by $D$. Further, since $P'\in\mathcal{P}_{t_e}$ we have $\|P_e^k(\cdot|s,a) - P(\cdot|s,a)\|_1\le \sqrt{\frac{14S\log(2AT)}{1\vee n_{t_e}(s,a)}}$ for all $s,a$. Hence, we have $B_{\pi_e}^{P_e^k,k}\le \min\{2D, D\sqrt{\frac{14S\log(2AT)}{1\vee n_{t_e}(s,a)}}\}$.

We now need to bound the expected value in Equation \eqref{eq:long_term_average_on_starting_distribution}. Note that conditioned on filtration $H_{t_e}$ the two expectations $\mathbb{E}_{s,a\sim\pi_e,P}[\cdot]$ and $\mathbb{E}_{s,a\sim\pi_e,P}[\cdot|H_{t_e-1}]$ are not equal as the former is the expected value of the long-term state distribution and the later is the long-term state distribution condition on initial state $s_{t_e-1}$. We now use Assumption \ref{bounded_mixing_time_assumtion} to obtain the following set of inequalities.
\begin{align}
    \mathbb{E}_{(s,a)\sim\pi_e, P}[B_{\pi_e}^{P_e^k, k}(s,a)] &= 
    \mathbb{E}_{(S_t, A_t)\sim\pi_e, P}[B_{\pi_e}^{P_e^k, k}(S_t, A_t)|H_{t_e-1}] \nonumber\\
    &~~+ \left(\mathbb{E}_{(s,a)\sim\pi_e, P}[B_{\pi_e}^{P_e^k, k}(s,a)]- \mathbb{E}_{(S_t, A_t)\sim\pi_e, P}[B_{\pi_e}^{P_e^k, k}(S_t, A_t)|H_{t_e-1}]\right)\\
    &\le\mathbb{E}_{(S_t, A_t)\sim\pi_e, P}[B_{\pi_e}^{P_e^k, k}(S_t, A_t)|H_{t_e-1}] \nonumber\\
    &~~+ 2D\left(\left|\sum_{a,s}\left(\pi_e(a|s)d_{\pi_e}(s) - \pi_e(a|s)P_{\pi,S_{t_e-1}}^{t-t_e+1}(s)\right)\right|\right)\\
    &=\mathbb{E}_{(S_t, A_t)\sim\pi_e, P}[B_{\pi_e}^{P_e^k, k}(S_t, A_t)|H_{t_e-1}] \nonumber\\
    &~~+ 2D\left(\left|\sum_{s}\left(\left(d_{\pi_e}(s) - P_{\pi,S_{t_e-1}}^{t-t_e+1}(s)\right)\left(\sum_{a}\pi_e(a|s)\right)\right)\right|\right)\\
    &=\mathbb{E}_{(S_t, A_t)\sim\pi_e, P}[B_{\pi_e}^{P_e^k, k}(S_t, A_t)|H_{t_e-1}] \nonumber\\
    &~~+ 2D\left(2\|d_{\pi_e}(s) - P_{\pi,S_{t_e-1}}^{t-t_e+1}(s)\|_{TV}\right)\label{eq:change_expectation_to_diff_prob}\\
    &\le\mathbb{E}_{(S_t, A_t)\sim\pi_e, P}[B_{\pi_e}^{P_e^k, k}(S_t, A_t)|H_{t_e-1}] + 4CD\rho^{t-t_e}\label{eq:TV_bounded_by_l1}
\end{align}
where Equation \eqref{eq:change_expectation_to_diff_prob} comes from Assumption \ref{bounded_mixing_time_assumtion} for running policy $\pi_e$ starting from state $s_{t_e-1}$ for $t-t_e+1$ steps and from Lemma \ref{lem:bound_bellman_s_a_main}. Equation \eqref{eq:TV_bounded_by_l1} follows from the fact that $\sum_a\pi(a|s) = 1$ and the fact that $\ell_1$-distance of probability distribution is twice the total variation distance \citep[{Proposition 4.2}]{levin2017markov}.

Using Equation \eqref{eq:TV_bounded_by_l1} with Equation \eqref{eq:long_term_average_on_starting_distribution} we get,

\begin{align}
    R_1^k(T) &\le \sum_{e=1}^E\mathbb{E}\left[\Big| \sum_{t=t_e}^{t_{e+1}-1}\mathbb{E}_{(s,a)\sim\pi_e, P}[B_{\pi_e}^{P, k}(s,a)]\Big|\Bigg|\mathcal{E}_{t_e} \cap \widehat{\mathcal{E}}_{t_e}\right]\\
    &\le \sum_{e=1}^E\mathbb{E}\left[\Big| \sum_{t=t_e}^{t_{e+1}-1}\mathbb{E}_{(S_t, A_t)\sim\pi_e, P}[B_{\pi_e}^{P, k}(S_t, A_t)|H_{t_e-1}] + 4CD\rho^{t-t_e}\Big|\Bigg|\mathcal{E}_{t_e} \cap \widehat{\mathcal{E}}_{t_e}\right]\\
    &= \sum_{e=1}^E\mathbb{E}\left[\Big| \sum_{t=t_e}^{t_{e+1}-1}\mathbb{E}_{(S_t, A_t)\sim\pi_e, P}[B_{\pi_e}^{P, k}(S_t, A_t)|H_{t_e-1}]\Big|\Bigg|\mathcal{E}_{t_e} \cap \widehat{\mathcal{E}}_{t_e}\right] + \sum_{e=1}^E\frac{4CD}{1-\rho}\label{eq:expected_diameter_deviation_from_steady_state}\\
    &= \sum_{e=1}^E\mathbb{E}\left[\Big| \sum_{t=t_e}^{t_{e+1}-1}\mathbb{E}_{(S_t, A_t)\sim\pi_e, P}[B_{\pi_e}^{P, k}(S_t, A_t)|H_{t_e-1}]\Big|\Bigg|\mathcal{E}_{t_e} \cap \widehat{\mathcal{E}}_{t_e}\right] + \frac{4ECD}{1-\rho}\label{eq:bound_on_restart_deviations}
\end{align}

where Equation \eqref{eq:expected_diameter_deviation_from_steady_state} follows from summation of series $\rho^{t-t_e}$ for $t\to\infty$.

We can now construct a Martingale sequence to bound the summation in Equation \eqref{eq:bound_on_restart_deviations}. We construct a Martingale sequence as 
\begin{align}
    X_t^e = \mathbb{E}_{(S_t, A_t)\sim\pi_e,P}[\sum_{t=t_e}^{t_{e+1}-1}B_{\pi_e}^{P_e^k,k}(S_t, A_t)|H_{t-1}]; t_e\le t<t_{e+1}
\end{align}
such that $|X_t^e-X_{t-1}^e| \le 4D$ for all $t,e$. We can now use Azuma-Hoeffding's inequality to bound $X_{t_e}^e$ as:

\begin{align}
    \Bigg|\left(X_{t_e}^e-\sum_{t=t_e}^{t_{e+1}-1}B_{\pi_e}^{P_e^k,k}(S_t, A_t)\right)\Bigg|\le 4D\sqrt{(t_{e+1}-t_e)\log(2/T)}
\end{align}
with probability at least $1-1/T$. This eventually gives an upper bound on $X_t^e$ as:
\begin{align}
    |X_t^e|-\Bigg|\sum_{t=t_e}^{t_{e+1}-1}B_{\pi_e}^{P_e^k,k}(S_t, A_t)\Bigg| &\le \Bigg|\left(X_t^e-\sum_{t=t_e}^{t_{e+1}-1}B_{\pi_e}^{P_e^k,k}(S_t, A_t)\right)\Bigg|\\
    &\le 4D\sqrt{(t_{e+1}-t_e)\log(2/T)}\\
    \implies |X_t^e| &\le \Bigg|\sum_{t=t_e}^{t_{e+1}-1}B_{\pi_e}^{P_e^k,k}(S_t, A_t)\Bigg| + 4D\sqrt{(t_{e+1}-t_e)\log(2/T)}\label{eq:bound_bellman_expected_episode}
\end{align}

Hence, for $n_{t_e}(s,a) = \sum_{t'=1}^{t_e-1}{\bf 1}_{\{S_{t'}=s,a_{t'}=a\}} = \sum_{e'=1}^{e-1}\sum_{t'=t_{e'}}^{t_{e'+1}-1}{\bf 1}_{\{S_{t'}=s,a_{t'}=a\}} = \sum_{e'=1}^{e-1}\nu_{e}(s,a)$, we have
\begin{align}
    R_1(T)&\le \sum_{e=1}^E\mathbb{E}\left[\Big|D\sum_{t_e}^{t_{e+1}-1}B_{\pi_e}^{P_e^k,k}(S_t, A_t)\Big| + \Big|4D\sqrt{(t_{e+1}-t_e)\log (2T)}\Big|\Bigg|\mathcal{E}_{t_e} \cap \widehat{\mathcal{E}}_{t_e}\right] + \frac{4ECD}{1-\rho}\label{eq:Azuma_Hoeffding_Bellman_error}\\
    &\le \sum_{e=1}^E\mathbb{E}\left[\Big|D\sum_{s,a}\nu_e(s,a)B_{\pi_e}^{P_e^k,k}(s,a)\Big| + \Big|4D\sqrt{(t_{e+1}-t_e)\log (2T)}\Big|\Bigg|\mathcal{E}_{t_e} \cap \widehat{\mathcal{E}}_{t_e}\right] + \frac{4ECD}{1-\rho}\\
    &\le \sum_{e=1}^E\mathbb{E}\left[\Big|D\sum_{s,a}\nu_e\sqrt{\frac{14S\log(2AT)}{n_{t_e}(s,a)}}\Big| + \Big|4D\sqrt{(t_{e+1}-t_e)\log (2T)}\Big|\Bigg|\mathcal{E}_{t_e} \cap \widehat{\mathcal{E}}_{t_e}\right] + \frac{4ECD}{1-\rho}\label{eq:pre_remove_modulus}\\
    &= \sum_{e=1}^E\mathbb{E}\left[D\sum_{s,a}\nu_e\sqrt{\frac{14S\log(2AT)}{n_{t_e}(s,a)}} + 4D\sqrt{(t_{e+1}-t_e)\log (2T)}\Bigg|\mathcal{E}_{t_e} \cap \widehat{\mathcal{E}}_{t_e}\right] + \frac{4ECD}{1-\rho}\label{eq:remove_modulus}\\
    &= D\sqrt{14S\log(2AT)}\mathbb{E}\left[\sum_{e=1}^E\sum_{s,a}\nu_e\frac{1}{\sqrt{n_{t_e}(s,a)}}\Big|\mathcal{E}_{t_e} \cap \widehat{\mathcal{E}}_{t_e}\right] + \mathbb{E}\left[\sum_{e=1}^E4D\sqrt{(t_{e+1}-t_e)\log (2T)}\Big|\mathcal{E}_{t_e} \cap \widehat{\mathcal{E}}_{t_e}\right] \nonumber\\
    &~~~~+ \frac{4ECD}{1-\rho}\\
    &\le D\sqrt{14S\log(2AT)}\sum_{s,a}(\sqrt{2}+1)\sqrt{N(s,a)} + 4D\sqrt{ET\log (2T)} + \frac{4ECD}{1-\rho}\label{eq:jaksch_sum_sqrt}\\
    &\le D(\sqrt{2}+1)\sqrt{14S\log(2AT)}\sqrt{SAT} + 4D\sqrt{ET\log (2T)} + \frac{4ECD}{1-\rho}\label{eq:Cauchy_Schwarz}
\end{align}
where Equation \eqref{eq:Azuma_Hoeffding_Bellman_error} comes from replacing $X_t^e$ of Equation \eqref{eq:bound_on_restart_deviations} by Equation \eqref{eq:bound_bellman_expected_episode}.
Equation \eqref{eq:remove_modulus} follows from the fact that both terms in modulus operator in Equation \eqref{eq:pre_remove_modulus} are non-negative.
In Equation \eqref{eq:jaksch_sum_sqrt} the first term follows from \citep{jaksch2010near}[Lemma 19] and second term follows from Cauchy-Schwarz inequality. Equation \eqref{eq:Cauchy_Schwarz} follows from Cauchy-Schwarz inequality.

We bound the $R_2^k(T)$ term similarly as:
\begin{align}
    R_2^k(T) &= \sum_{e=1}^E\mathbb{E}\left[\Big| \sum_{t=t_e}^{t_{e+1}-1}\left(\lambda_{\pi_e}^{P,k} - r^k(S_t, A_t)\right)\Big|\Bigg|\mathcal{E}_{t_e} \cap \widehat{\mathcal{E}}_{t_e}\right]\\
    &\le 2\sqrt{ET\log(2T)} + \frac{4EC}{1-\rho}
\end{align}
with probability $1/T$.

Now, combining the regret using the law of total expectation, we have:
\begin{align}
\sum_{e=1}^E\mathbb{E}\left[\Big| \sum_{t=t_e}^{t_{e+1}-1}\left(\lambda_{\pi_e}^{P_e, k} -r^k(S_t, A_t)\right)\Big|\right] &= \sum_{e=1}^E\mathbb{E}\left[\Big| \sum_{t=t_e}^{t_{e+1}-1}\left(\lambda_{\pi_e}^{P_e, k} -r^k(S_t, A_t)\right)\Big|\Bigg| \mathcal{E}_t^c\cup\widehat{\mathcal{E}}_t^c\right]\mathbb{P}\left(\mathcal{E}_t^c\cup\widehat{\mathcal{E}}_t^c\right) \nonumber\\
&~~+ \sum_{e=1}^E\mathbb{E}\left[\Big| \sum_{t=t_e}^{t_{e+1}-1}\left(\lambda_{\pi_e}^{P_e, k} -r^k(S_t, A_t)\right)\Big|\Bigg| \mathcal{E}_t\cap\widehat{\mathcal{E}}_t\right]\mathbb{P}\left(\mathcal{E}_t\cap\widehat{\mathcal{E}}_t\right)\\
&\le \sum_{e=1}^E\mathbb{E}\left[\Big| \sum_{t=t_e}^{t_{e+1}-1}\left(\lambda_{\pi_e}^{P_e, k} -r^k(S_t, A_t)\right)\Big|\Bigg| \mathcal{E}_t^c\cup\widehat{\mathcal{E}}_t^c\right] \nonumber\\
&~~+ \sum_{e=1}^E\mathbb{E}\left[\Big| \sum_{t=t_e}^{t_{e+1}-1}\left(\lambda_{\pi_e}^{P_e, k} -r^k(S_t, A_t)\right)\Big|\Bigg| \mathcal{E}_t\cap\widehat{\mathcal{E}}_t\right]\label{eq:prob_max_1}\\
&\le 4\sqrt{T} + (\sqrt{2} + 1)DS\sqrt{14AT\log(2AT)} \nonumber\\
&~~~+4D\sqrt{ET\log (2T)} + 2\sqrt{ET\log(2T)} + \frac{4EC(D+1)}{1-\rho}
\end{align}
where Equation \eqref{eq:prob_max_1} follows from the fact that probability for any event can be at most $1$.

Combining everything gives the regret:
\begin{align}
    R(T) = \frac{LK}{T}\left(4\sqrt{T} + (\sqrt{2} + 1)DS\sqrt{14AT\log(2AT)} + (4D+2)\sqrt{ET\log (2T)} + \frac{4EC(D+1)}{1-\rho}\right)
\end{align}
and completes the proof from using the upper bound of $E$ from Lemma \ref{lem:bound_on_episodes}.

\end{proof}

\color{black}

\section{Model Free Algorithm} \label{mfa}
In the previous section, we developed a model based tabular algorithm for joint function optimization. However, as the state space, action space, or number of agents increase the tabular algorithm becomes infeasible to implement. In this section, we consider a policy gradient based algorithm which can be efficiently implemented using (deep) neural networks  thus alleviating the requirement of a tabular solution for large MDPs.

For the model free policy gradient algorithm, we will use finite time horizon MDP, or $T < \infty$ in our MDP $\mathcal{M}$. This is a practical scenario where communication networks optimize fairness among users for finite duration \citep{margolies2016exploiting}. We now describe a model free construction to obtain the optimal policy. %
We use a neural network parameterized by $\theta$. The objective thus becomes to find optimal parameters $\theta^*$, which maximizes,
\begin{align}
    \arg\max_{\theta}f\left((1-\gamma)J_{\pi_\theta}^1, \cdots, (1-\gamma)J_{\pi_\theta}^K\right).\label{eq:parameterized_objective}
\end{align}
For model-free algorithm, we assume that the function is differentiable. In case, the function is not differentiable, sub-gradients of $f$ can be used to optimize the objective \citep{nesterov2018lectures}.
Gradient estimation for Equation \eqref{eq:parameterized_objective} can be obtained using chain rule:
\begin{align}
    \nabla_{\theta}f &= \sum_{k\in [K]}(1-\gamma)\frac{ \partial f}{ \partial (1-\gamma)J_{\pi}^k}\nabla_{\theta} J_{\pi}^k\label{eq:policy_gradient}
\end{align}

For all $k$, $J_\pi^k$ can be replaced with averaged cumulative rewards over $N$ trajectories for the policy at $i^{th}$ step, where a trajectory $\tau$ is defined as the tuple of observations, or $\tau = (s_0, a_0, r_0^1, \cdots, r_0^K, s_1, a_1, r_1^1, \cdots, r_1^K, \cdots)$. Further, $\nabla_{\theta} J_{\pi}^k$ can be estimated using  REINFORCE algorithm  proposed in \citep{williams1992simple,sutton2000policy} for all $k$, and is given as
\begin{align}
    \widehat{\nabla}_{\theta} J_{\pi}^k = \frac{1}{N}\sum_{j=1}^N\sum_{t=0}^T\nabla_{\theta}\log\pi_{\theta}(A_{t,j}|S_{t,j})\sum_{t'=t}^Tr^k(S_{t', j}, A_{t', j})~\forall~k\in[K].\label{eq:reinforce_grad}
\end{align}
Further, $J_\pi^k $ is estimated as
\begin{align}
 \hat{J}_\pi^k = \frac{1}{N}\sum_{t=0}^Tr^k(S_{t, j}, A_{t, j})~\forall~k\in[K].\label{eq:J_estimate}
 \end{align}
For a learning rate $\eta$, parameter update step to optimize the parameters becomes
\begin{align}
    \theta_{i+i} = \theta_i + \eta (1-\gamma)\sum_{k\in[K]}\frac{ \partial f}{ \partial (1-\gamma)J_{\pi}^k}\Big|_{J_\pi^k = \hat{J}_\pi^k}\widehat{\nabla}_{\theta} J_{\pi}^k\label{eq:Gradient_ascent}
\end{align}

\color{black}
For example, consider alpha-fairness utility defined as

\begin{align}
    f((1-\gamma)J_\pi^1, \cdots, (1-\gamma)J_\pi^K) = \sum_{k=1}^K \frac{1}{1-\alpha}\left((J_\pi^k)^{1-\alpha} - 1\right)\label{eq:alpha_fair_def}
\end{align}

 Then the corresponding gradient estimate can be obtained as
\begin{align}
	\hat{\nabla}_{\theta} f = \sum_{k\in [K]} \frac{\sum\limits_{j=1}^N\sum\limits_{t=0}^T\nabla_{\theta}\log\pi(A_{t,j}|S_{t,j})\sum\limits_{\tau=t}^T \gamma^{\tau}r_k(S_{\tau,j}, A_{\tau,j})} {(1-\gamma)^{\alpha-1}\left(\sum\limits_{j= 1}^N\sum\limits_{t=0}^T \gamma^t r_k(S_{t, j}, A_{t, j})\right)^{\alpha}}. \label{eq:alpha_fair_grad}  
\end{align}
 
\color{black}

The proposed Model Free Policy Gradient algorithm for joint function optimization is described in Algorithm \ref{alg:model_free_PG}. The algorithm takes as input the parameters $\mathcal{S}, \mathcal{A}, [K], T, \gamma, f$ of MDP $\mathcal{M}$, number of sample  trajectories $N$, and learning rate $\eta$ as input. The policy neural network is initialized with weights $\theta$ randomly. It then collects $N$ sample trajectories using the policy with current weights in Line 4. In Line 5, the gradient is calculated using Equations \eqref{eq:policy_gradient}, \eqref{eq:reinforce_grad}, and \eqref{eq:J_estimate} on the $N$ trajectories. In optimization step of Line 6, the weights are updated using gradient ascent.
\begin{algorithm}[!htbp]
    \begin{small}
	\caption{Model Free Joint Policy Gradient}\label{alg:model_free_PG}
    \begin{algorithmic}[1]
        \Procedure{Joint Policy Gradient  }{$\mathcal{S}, \mathcal{A}, [K], T, \gamma, f, N, \eta$}
            \State Initialize $\pi_{\theta_0}(a,s)$ \Comment{Initialize the neural
            network with random weights $\theta$}
            \For{$i = 0, 1, \cdots,\ until\ convergence\ criteria$}
                \State Collect $N$ trajectories using policy $\pi_{\theta_i}$
                \State Estimate gradient using Equation \eqref{eq:policy_gradient}, \eqref{eq:reinforce_grad}, \eqref{eq:J_estimate}
                \State Perform Gradient Ascent using Equation \eqref{eq:Gradient_ascent}
            \EndFor
            \State Return $\pi_{\theta}$
        \EndProcedure
    \end{algorithmic}
    \end{small}
\end{algorithm}
\section{Evaluations} \label{simulations}

In this section, we consider two systems. The first is the cellular scheduling, where multiple users connect to the base station. The second is a multiple-queue system which models multiple roads merging into a single lane. In both these systems, the proposed algorithms are compared with some baselines including the linear metric adaptation of reward at each time. 

\subsection{Cellular fairness maximization}
The fairness maximization has  been at the heart of many other resource allocation problems such as cloud resource management, manufacturing optimization, etc. \citep{perez2009responsive,zhang2015fairness}. The problem of maximizing wireless network fairness has been extensively studied in the past  \citep{margolies2016exploiting,kwan2009proportional,bu2006generalized,li2018resource}. With increasing number of devices that need to access wireless network and ever upgrading network architectures, this problem still remains of practical interest. We consider two problem setup for fairness maximization, one with finite state space, and other with infinite state space. For finite state space, we evaluate both the proposed model-based algorithm (Algorithm \ref{alg:model_based_algo}) and the proposed model-free algorithm (Algorithm \ref{alg:model_free_PG}) while for infinite state space, we evaluate only model-free algorithm as tabular model-based algorithm cannot be implemented for this case.

\subsubsection{Problem Setup}
We consider fairness metric of the form of generalized $\alpha$-fairness proposed in \citep{altman2008generalized}. For the rest of the paper, we will call this metric as $\alpha$-fairness rather than generalized $\alpha$-fairness. The problem of maximizing finite horizon $\alpha$-fairness for multiple agents attached to a base station is defined as
\begin{align}
 C_{\alpha}(T) = &   \max_{\{a_{k,t}\}_{K\times T}}\  \sum_{k=1}^K\frac{1}{1-\alpha}\left(\left(\frac{1}{T}\sum_{t=1}^T a_{k,t}r_{k,t}\right)^{(1-\alpha)} - 1\right)\label{eq:alpha_fair_eq}\\
  &  \text{s.t.} \sum_{k=1}^Ka_{k,t} = 1 \forall t\in\{1, 2, \cdots, T\}\\
    &\quad a_{i,j} \in {0, 1}
\end{align}
where, $a_{k,t} = 1$ if the agent $k$ obtains the network resource at time $t$, and $0$ otherwise. This implies, at each time $t$, the scheduler gives all the resources to only one of the attached user. Further, $r_{k, t}$ denotes the rate at which agent $k$ can transmit at time $t$ if allocated network resource. {\color{black}Cellular networks typically use Proportional Fair (PF) utility which is a special case of the above metric for $\alpha \to 1$ \citep{holma2005wcdma}, and is defined as:}
\begin{align}
 C_1(T) = &   \max_{\{a_{k,t}\}_{K\times T}}\  \sum_{k=1}^K\log\left(\frac{1}{T}\sum_{t=1}^T a_{k,t}r_{k,t}\right)\label{eq:prop_fair_eq}
\end{align}
We note that $r_{k, t}$ is only known causally limiting the use of offline optimization techniques and making the use of learning-based strategies for the problem important. We evaluate our algorithm for $\alpha=2$ fairness, and proportional fairness for $T = 1000$.

\subsubsection{Proportional Fairness}
We let the number of users attached to the network, $K$, belong to the set $\{2, 4, 6\}$. The state space of each agent comes from its channel conditions. We assume that the channel for a agent can only be in two conditions $\{good,\ bad\}$, where the good and bad conditions for each agent could be different. The action at each time is a one-hot vector with the entry corresponding to the agent receiving the resources set to one. This gives $|\mathcal{S}| = 2^K$ (corresponding to the joint channel state of all agents), and $|\mathcal{A}| = K$ ($K$ actions correspond to the agent that is selected in a time slot). Based on the channel state of agent, the scheduling decision determines the agent that must be picked in the time-slot. Rate $r_{k,t}$, for agent $k$ at time $t$, is dependent on  the state of the agent $s_{k,t}$ and is mentioned in Table \ref{rate_table}.
Each agent remains in the same state with probability of $0.8$, and moves to a different state $s\sim U(\mathcal{S})$ with probability $0.2$. The state transition model becomes,
\begin{small}
\begin{align}
    \forall\ k,\ s_{k, t+1}= \begin{cases}
                s_{k,t}, &w.p.\ 0.8\\
                s\sim U(\{good, bad\}) &w.p.\ 0.2
            \end{cases}
\end{align}
\end{small}

\begin{table}[!htbp]
\centering
\begin{tabular}{|c |c |c | c| c| c| c|} 
 \hline
 Agent state & $r_{1,t}$ & $r_{2,t}$ & $r_{3,t}$ & $r_{4,t}$ & $r_{5,t}$ & $r_{6,t}$\\ [0.5ex] 
 \hline
 $good$ &1.50 & 2.25 & 1.25 &1.50 & 1.75 & 1.25 \\ [0.5ex] 
 $bad$ &0.768 & 1.00 & 0.384 &1.12 & 0.384 & 1.12\\ [0.5ex] 
 \hline
\end{tabular}
\caption{Agent rate $r_{k, t}$ (in Mbps) based on agent state $s_{k,t}$. Rate values are practically observable data rates over a wireless network such as 4G-LTE.}
\label{rate_table}
\end{table}

We compare our model-based and model-free algorithms with practically implemented algorithm of Blind Gradient Estimation \citep{margolies2016exploiting,bu2006generalized} in network schedulers and SARSA based algorithm devised by \citep{perez2009responsive}. We first describe the algorithms used in evaluations for maximizing proportional fairness in finite state systems.
\begin{itemize}[leftmargin=*]
    \item {\bf Blind Gradient Estimation Algorithm (BGE)}: This heuristic allocates the resources based on the previously allocated resources to the agents. Starting from $t=1$, this policy allocates resource to  agent $k^*_t$ at time $t$, where
    \begin{align}
        k^*_t = \arg\max_{k\in[K]} \frac{r_{k,t}}{\sum_{t'=0}^{t-1}\alpha_{k, t'}r_{k, t'}}; \alpha_{k, t'} = \begin{cases}
                1, & k= k^*_{t'}\\
                0, & k\neq k^*_{t'}
            \end{cases}
    \end{align}

    BGE is used as de facto standard for scheduling in cellular systems \citep{holma2005wcdma}, and has been shown to be asymptotically optimal for the proportional fairness metric \citep{1003822}.
    
    \item {\bf DQN Algorithm}: This algorithm based on SARSA \citep{sutton1998introduction} and DQN \cite{mnih2015human}. The reward at each time $\tau$ is the fairness of the system at time $\tau$, or $f_\tau = C_{\alpha}(t)$ 
    The DQN neural network consists of two fully connected hidden layers with $100$ units each with ReLU activation and one output layer with linear activation. We use $\gamma = 0.99$, $\epsilon = 0.05$, and Adam optimizer with learning rate $\eta = 0.01$ to optimize the network. The batch size is $64$ and the network is trained for $1500$ episodes.
    
    \color{black}
    \item {\bf Vanilla Policy-Gradient Algorithm}: This algorithm is based on the REINFORCE policy gradient algorithm \citep{williams1992simple}. Similar to the SARSA algorithm, the reward at each time $\tau$ is the fairness of the system at time $\tau$, or $f_\tau = C_{\alpha}(t)$ 
    We use $\gamma = 0.99$, and learning rate $\eta = 1\times 10^{-4}$.
    \color{black}
    
    \item {\bf Proposed Model Based Algorithm}: We describe the algorithm for infinite horizon, so we maximize the policy for infinite horizon proportional fairness problem by discounting the rewards as
    
    \begin{align}
        \lim_{T\to\infty}\sum_{k=1}^K\log\left(\frac{1}{T}\sum_{t=1}^T\gamma^t\alpha_{k,t}r_{k,t}\right)
    \end{align}
    The learned policy is evaluated on finite horizon environment of $T=1000$. We keep $\gamma = 0.99$ for implementation of Algorithm \ref{alg:model_based_algo}. We use a fixed episode length of $100$ and update policy after every $\tau = 100$ steps. In Algorithm \ref{alg:model_based_algo}, the convex optimization is solved using CVXPY \cite{diamond2016cvxpy}.

    \item {\bf Proposed Model Free Algorithm}: %
     Since $\log(\cdot)$ is differentiable, the gradient in Equation \eqref{eq:policy_gradient} is evaluated using Equation \eqref{eq:alpha_fair_grad} at $\alpha=1$. The neural network consists of a single hidden layer with $200$ neurons, each having ReLU activation function. The output layer uses softmax activation. The value of other hyperparameters are $\gamma=0.99$, $\eta = 1\times10^{-3}$, and batch size $N=100$. The algorithm source codes for the proposed algorithms have been provided at \citep{code_nlmarl}.
\end{itemize}

\begin{figure*}[tbhp]
	\centering
	\subfigure[$K=2$]{ 
		\includegraphics[trim=0in .1in .4in .2in, clip, width=.45\textwidth]{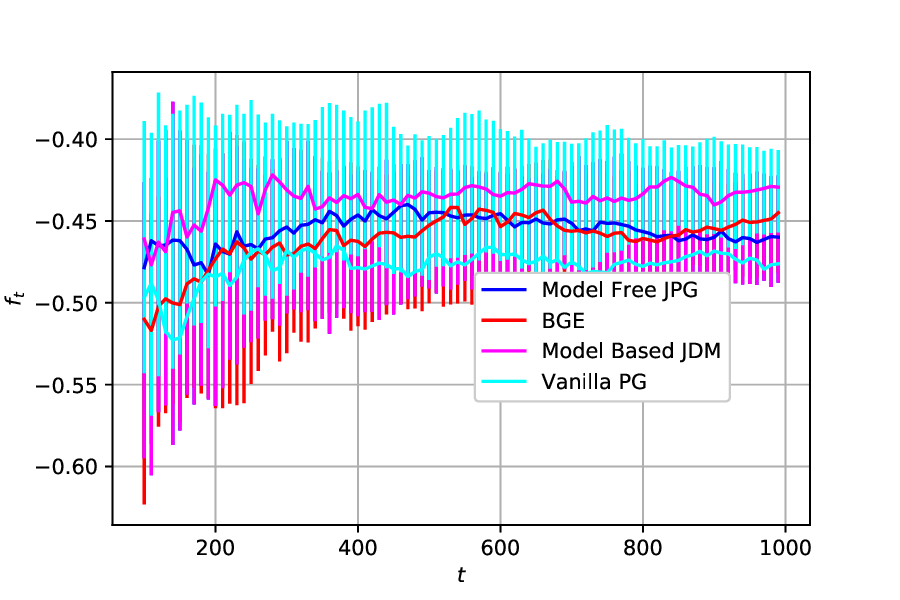}
		\label{fig:fig_k_2}}
	\subfigure[$K=2$ (with DQN)]{ 
		\includegraphics[trim=0in .1in .4in .2in, clip, width=.45\textwidth]{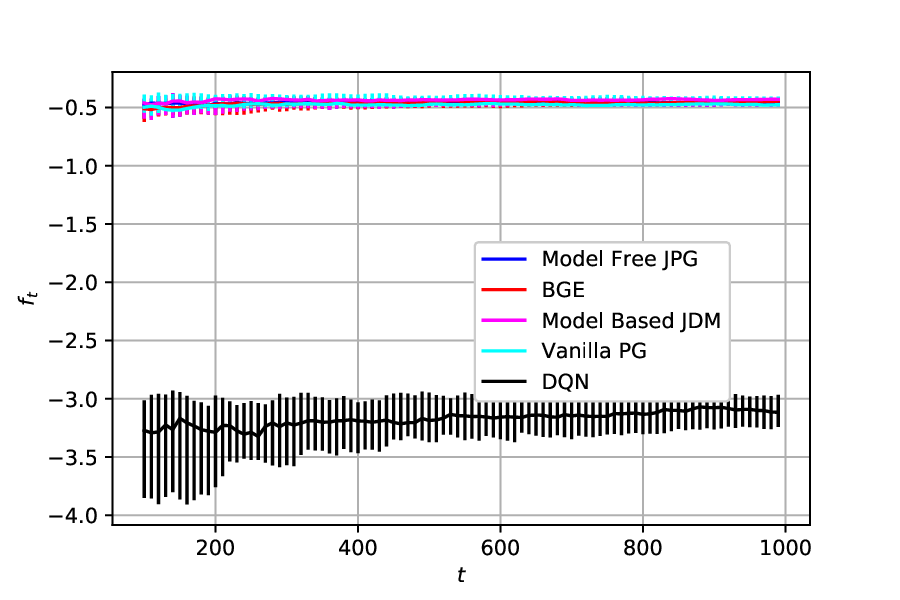}
		\label{fig:fig_k_2_dqn}}
	\subfigure[$K=4$]{ 
		\includegraphics[trim=0in .1in .4in .2in, clip, width=.45\textwidth]{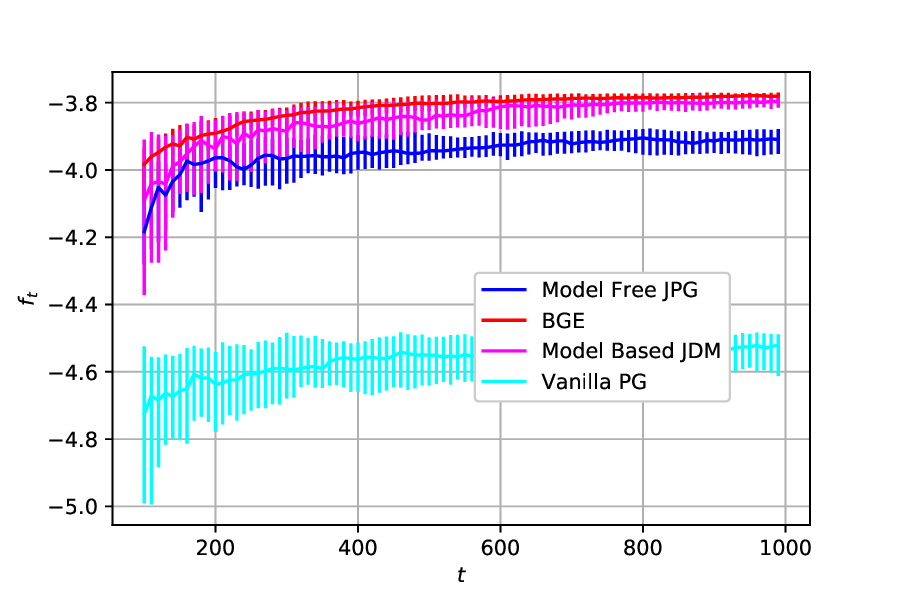}
		\label{fig:fig_k_4}}
	\subfigure[$K=4$ (with DQN)]{ 
		\includegraphics[trim=0in .1in .4in .2in, clip, width=.45\textwidth]{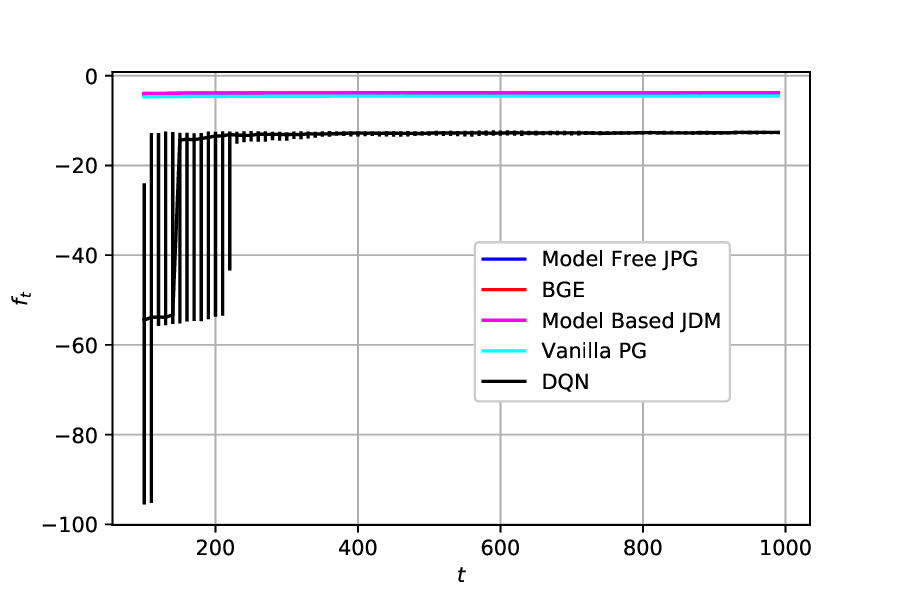}
		\label{fig:fig_k_4_dqn}}
	\subfigure[$K=6$]{ 
		\includegraphics[trim=0in .1in .4in .2in, clip, width=.45\textwidth]{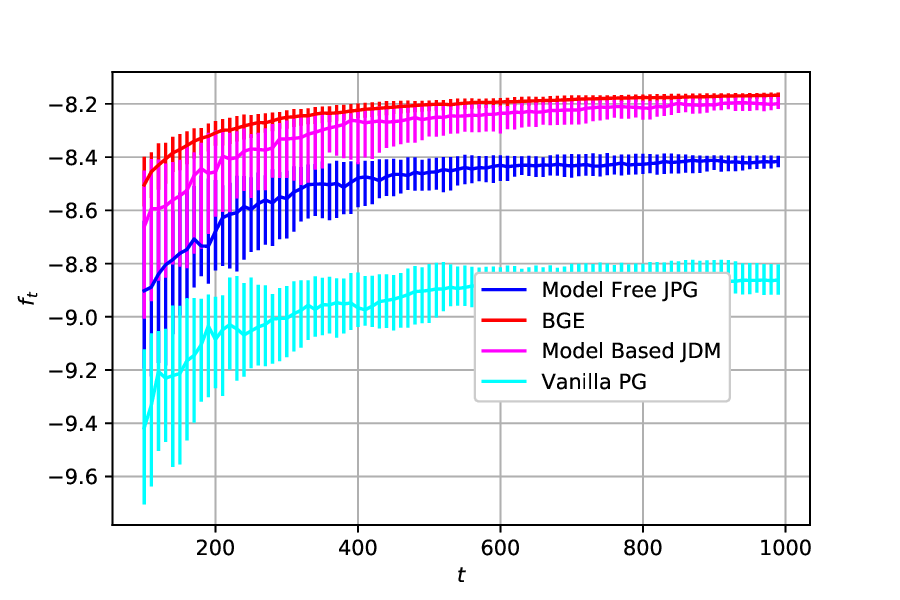}
		\label{fig:fig_k_6}}
	\subfigure[$K=6$ (with DQN)]{ 
		\includegraphics[trim=0in .1in .4in .2in, clip, width=.45\textwidth]{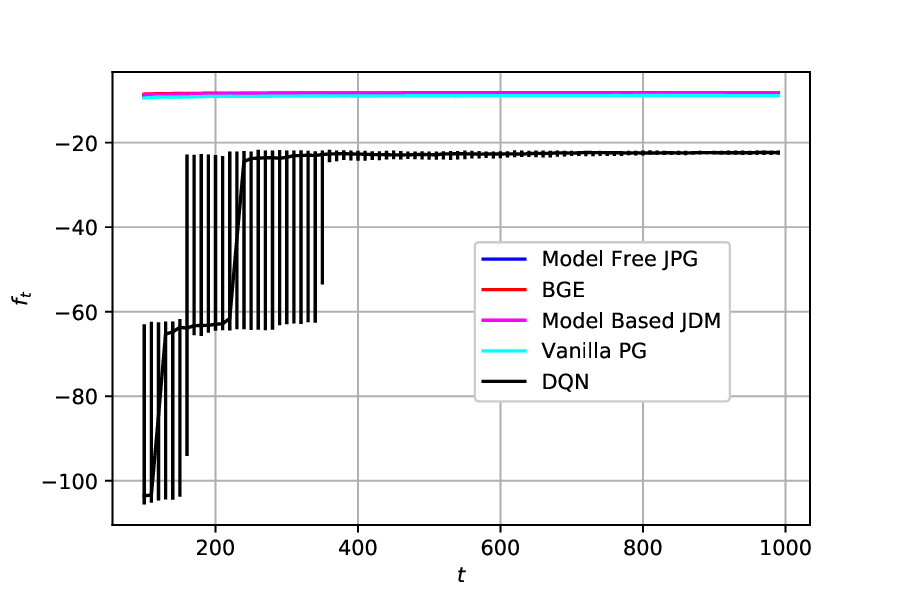}
		\label{fig:fig_k_6_dqn}}		
	\caption{\small Proportional Fairness for Cellular Scheduling,  $f_t$ v/s $t$ (Best viewed in color)}
	\label{fig:final_fig}
\end{figure*}

{\bf Proportional Fairness Simulation Results:}
We trained the SARSA algorithm and the model based Algorithm \ref{alg:model_based_algo} for $5000$ time steps for each value of $K$. 
{\color{black}To train model free Algorithm \ref{alg:model_free_PG} and the Vanilla Policy Gradient algorithm, we used $1000$ batches where each batch contains $36$ trajectories of length $1000$ time steps.} Note that the Blind Gradient Estimation algorithm doesn't need training as it selects the agent based on observed rewards. For all the algorithms we performed a grid search to find the hyperparameters.

We show the performance of policies implemented by each of the algorithm. Each policy is run $50$ times and median and inter-quartile range is shown in Figure \ref{fig:final_fig} for  each policy. The policy performance for $K=2$, $K=4$, and $K=6$ is shown in Figure \ref{fig:fig_k_2}, Figure \ref{fig:fig_k_4}, and Figure \ref{fig:fig_k_6}, respectively. 

We note that the performance of model-based algorithm (Algorithm \ref{alg:model_based_algo}) and that  of the model-free algorithm (Algorithm \ref{alg:model_free_PG}) are close. For $K=2$, the model-based algorithm outperforms the BGE algorithm. For $K\in \{4, 6\}$, the gap between the model based algorithm and BGE algorithm is because of the finite time horizon. The proposed framework assumes an infinite horizon framework, but the algorithm is trained for finite time horizon. The regret of  $\tilde{O}\left(LDKS\sqrt{\frac{A}{T}}\right)$ also guarantees that the proposed algorithm converges towards optimal policies for large $T$.

\color{black}
From Figure \ref{fig:fig_k_2_dqn}, Figure \ref{fig:fig_k_4_dqn}, and Figure \ref{fig:fig_k_6_dqn}, we note that the DQN algorithm performs much worse than expected. The reason for this is that the joint objective function of fairness is non-linear and is not properly modelled with standard RL formulation. Also, from Figure \ref{fig:fig_k_2}, we note that for $K=2$, policy gradient algorithm which uses fairness till time $t$ can still learn a a good policy, but the performance is still not at par with the proposed framework. 
This is because, using the value of the joint objective as reward works as a linear approximation of the joint reward function. Hence, if the approximation is worse, the policy gradient algorithm with joint objective as reward will not be able to optimize the true reward function. Note that as $K$ increases, the  approximation error $\sum_{k}\log(\lambda_k)-\lambda_k$ increase. In the next experiment, we will demonstrate that if the approximation error is too large, the policy gradient algorithm can perform even worse than the DQN algorithm.
\color{black}

\subsubsection{$\alpha=2$ Fairness}
We now evaluate our algorithm with the metric of $\alpha$-fairness, where no optimal baseline is known. We also consider a large state-space to show the scalability of the proposed model-free approach.  We consider a Gauss-Markov channel model \citep{ariyakhajorn2006comparative} for modeling the channel state to the different users, and let the number of users be $K=8$. Under Gauss-Markov Model, channel state of each user $k$ varies as,
\begin{align}
    X_{k, t} = \sqrt{1-\beta^2}X_{k, t-1} + \beta \epsilon_t,\ \ \epsilon_t\sim\mathcal{N}(0,1).
\end{align}
We assume $X_{k, 0} \sim \mathcal{N}(0,1)$ for each $k\in[K]$. The rate for each user $i$ at time $t$ and in channel state $X_{i, t}$ is given as,
\begin{align}
    r_{k, t} = P_k|X_{k,t}|^2,
\end{align}
where $P_k$ is multiplicative constant for the $k^{th}$ user, which indicates the average signal-to-noise ratio to the user. We let $P_k = k^{-0.2}$.

Since the state space is infinite, we only evaluate  the model free algorithm. The gradient update equation is defined in Equation (\ref{eq:alpha_fair_grad}) with $\alpha=2$. The neural network consists of a single hidden layer with $200$ neurons, each having ReLU activation function. We use stochastic gradient ascent with learning rate $\eta=1\times10^{-3}$ to train the network. The value of other hyperparameters are $\gamma=0.99$, and batch size $N=36$. The network is trained for $1000$ epochs. 

\color{black}
Also, since no optimal baseline is known, we compare the model free algorithm with the Deep Q-Network (DQN) algorithm \citep{mnih2015human} and the Policy Gradient algorithm \cite{williams1992simple}. Reward for both DQN algorithm and the Policy Gradient algorithm at time $\tau$ is taken as the value $C_2(\tau)$. For DQN, the neural network consists of two fully connected hidden layers with $100$ units each with ReLU activation and one output layer with linear activation. We use Adam optimizer with learning rate $0.01$ to optimize the DQN network. The batch size is $64$ and the network is trained for $1500$ episodes. For the Policy Gradient algorithm, we choose a single hidden layer of $200$ neurons, with learning rate $1\times10^{-5}$. Similar to the implementation of the Algorithm \ref{alg:model_free_PG}, we select batch size $N=36$ and train the network for $1000$ epochs.
\color{black}

{\bf $\alpha$-fairness Simulations Results}
The results for $\alpha$-fairness are provided in Figure \ref{fig:final_fig_alpha}. {\color{black} Each policy is run $50$ times and median and we show the inter-quartile range for each policy.} As a baseline, we also consider a strategy that chooses a user in each time uniformly at random, this strategy is denoted as ``uniform random" (Figure \ref{fig:fig_alpha_JPG} )
\color{black}
We note that the DQN algorithm and the policy gradient algorithm are not able to outperform the uniform random policy while the proposed model free policy outperforms the uniform random policy. Further Figure \ref{fig:fig_alpha_standard} shows the detrimental effect of using incorrect gradients. The linear approximation has larger error for $\alpha=2$ joint objective than compared to the proportional fairness joint objective. The standard policy gradient with fairness as rewards now performs as worse as the DQN algorithm.
\color{black}

\begin{figure}[thbp]
	\begin{center}
	\subfigure[$K=2$]{ 
		\includegraphics[trim=0in .1in .4in .2in, clip, width=.45\textwidth]{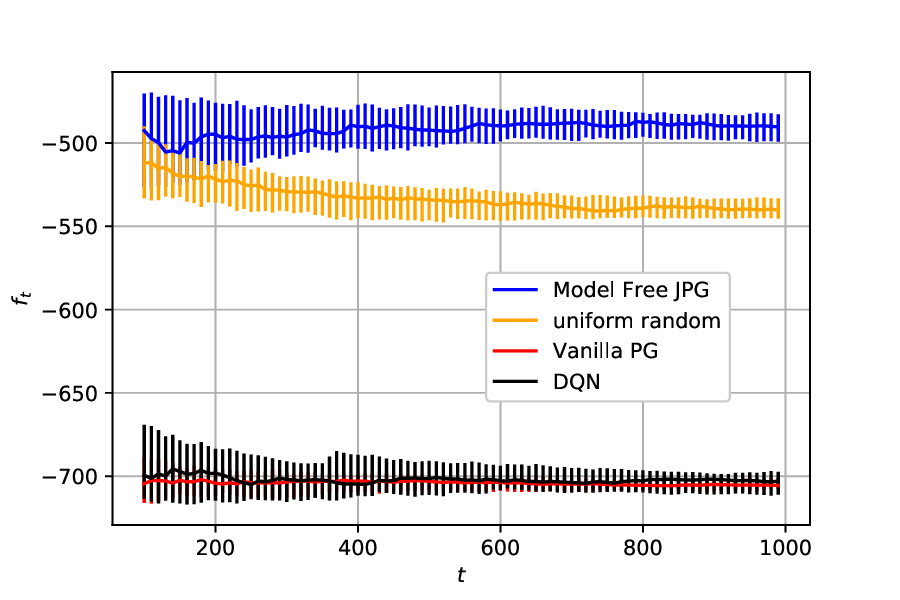}
		\label{fig:fig_alpha_JPG}}
	\subfigure[$K=2$ (standard RL algorithms)]{ 
		\includegraphics[trim=0in .1in .4in .2in, clip, width=.45\textwidth]{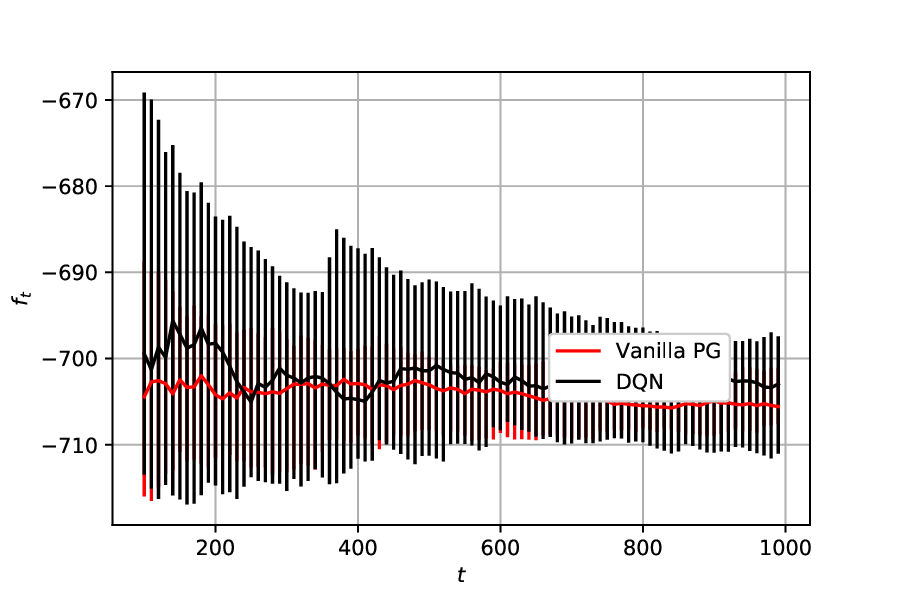}
		\label{fig:fig_alpha_standard}}
	\caption{\small Alpha Fairness for Cellular Scheduling, $f_t$ v/s $t$ (Best viewed in color)}
	\label{fig:final_fig_alpha}
\end{center}
\end{figure}

\subsection{Multiple Queues Latency Optimization}

We consider a problem, where multiple roads merge into a single lane, which is controlled by a digital sign in front of each road to indicate which road's vehicle proceeds next. This problem can be modeled as having $K$ queues with independent arrival patterns, where the arrival at queue $k\in \{1, \cdots, K\}$ is Bernoulli with an  arrival rate of $\lambda_i$. At each time, the user at the head of the selected one of out these $K$ queues is served. The problem is to determine which queue user is served at each time. Such problems also occur in processor scheduling systems, where multiple users send their requests to the server and the server decides which user's computation to do next \citep{haldar1991fairness}. 

In such a system, latency is of key concern to the users. The authors of \citep{ZhangE2E} demonstrated that the  effect of latency on the Quality of Experience (QoE) to the end user is a ``sigmoid-like'' function. Thus, we define the latency of a user $w$ as
\begin{align}
    QoE(w) = \frac{1-e^{-3}}{1+e^{(w-3)}}\label{eq:wait_time_qoe}.
\end{align}
Note that in Equation \eqref{eq:wait_time_qoe}, $QoE(w)$ remains close to $1$ for small wait times $(w\leq 1)$, and close to $0$ for high wait times $(w \geq 10)$. We let the service distribution of the queue be deterministic, where each user takes one unit of time for service. We model the problem as a non-linear multi-agent system. The different queues are the agents. The state is the queue lengths of the different queues. The action at each time is to determine which of the non-empty queue is selected. The reward of each agent at time $t$, $r_{k,t}$ is zero if queue $k$ is not selected at time $t$, and if the $QoE$ of the latency of the user served, if queue $k$ is selected at time $t$, where the latency of a user $w$ is the time spent by a user $w$ in the system (from entering the queue to being served). We again consider $\alpha$-fairness for $\alpha=2$, and proportional fairness for comparisons. %

\subsubsection{$\alpha=2$ Fairness}
We consider $K=8$ queues in our system. We let the arrival rate $\lambda_i$ in each queue be as given in  Table \ref{arrival_rate_table}. We also assume that each queue has bounded capacity of $100$, and the user is dropped if the queue is full. The overall reward function among different agents is chosen as $\alpha$-fairness, where $\alpha=2$. The objective can be written as
\begin{align}
C_{2}(T) = &   \max_{\{a_{k,t}\}_{K\times T}}\  \sum_{k=1}^K-1\left(\frac{1}{T}\sum_{t=1}^T a_{k,t}r_{k,t}\right)^{-1} \label{eq:lat_alpha_fair_eq}
\end{align}
\begin{table}[!htbp]
	\centering
\begin{tabular}{|c |c |c | c| c| c| c| c|} 
 \hline
 $\lambda_1$ & $\lambda_2$ & $\lambda_3$ & $\lambda_4$ & $\lambda_5$ & $\lambda_6$ & $\lambda_7$ &$\lambda_8$\\ [0.5ex] 
 \hline
 0.2& 0.1& 0.05& 0.25& 0.15& 0.21& 0.01& 0.3\\ [0.5ex] 
 \hline
\end{tabular}
\caption{Arrival rates $\lambda_{k}$ (in number of packets per step) for $\alpha=2$ fairness}
\label{arrival_rate_table}
\end{table}

\if 0
We consider the overall objective function as $\alpha$-fairness among the users over a certain period of time. We use $\alpha=2$, and the length of period $T=1000$. The objective can be written as
\begin{align}
    C_{2}(T) = &   \max_{\{a_{k,t}\}_{K\times T}}\  \sum_{k=1}^K-1\left(\frac{1}{T}\sum_{t=1}^T a_{k,t}r_{k,t}\right)^{-1} \label{eq:lat_alpha_fair_eq}
\end{align}
where $r_{k,t}$ is the $QoE$ of the corresponding to the wait time of the packet served at time $t$.
\fi

Since the number of states is large, we only evaluate  the model-free algorithm with $T=1000$. The gradient update equation for policy gradient algorithm as given in  Equation \eqref{eq:alpha_fair_grad} is used for $\alpha=2$.  Stochastic Gradient Ascent with learning rate $\eta=5\times10^{-3}$ is used to train the network. The value of discount factor $\gamma$ is set to $0.99$ and the batch size $N$ is kept as $30$.

We compare the proposed algorithm with the DQN Algorithm for Q-learning implementation. We use fairness at time $t$ or $C_2(t)$ as the reward for DQN network. The network consists of two fully connected hidden layers with $100$ units each,  ReLU activation function, and one output layer with linear activation. Adam optimizer with learning rate $0.01$ is used to optimize the network. The batch size is $64$ and the network is trained for $500$ episodes.

We also compare the proposed algorithm with the {\it Longest Queue First} (LQF) Algorithm, which serves the longest queue of the system. This algorithm doesn't require any learning, and has no hyperparameters.

We again run each policy for $50$ times and median and we show the inter-quartile range for each policy. The results for fairness maximization for this queuing system are provided in Figure (\ref{fig:final_fig_alpha_latency}). We note that the overall objective decreases  for all the policies. This is because the queue length is increasing and each packet has to wait for longer time on an average to be served till the queue becomes steady. At the end of the episode, the proposed policy gradient algorithm outperforms both the DQN and the LQF policy. We note that the during the start of the episode, LQF is more fair because the queues are almost empty, and serving the longest queue would decrease the latency of the longest queue. However, serving the longest queue is not optimal in the steady state. 

\begin{figure}[thbp]
\begin{center}
		\includegraphics[trim=0in .1in .4in .2in, clip,width=.48\textwidth]{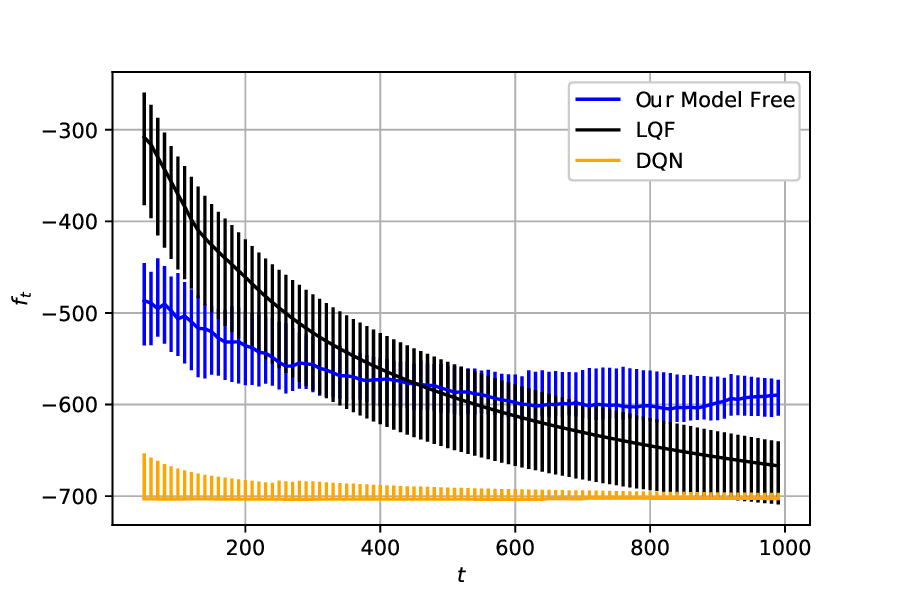}
	\caption{\small Alpha Fairness for the Queueing System, $f_t$ v/s $t$ (Best viewed in color)}
	\label{fig:final_fig_alpha_latency}
\end{center}
\end{figure}

\subsubsection{Proportional Fairness}

We also compare the performance of the policy learned by our algorithm with the policy learned by DQN algorithm for another synthetic example where we reduce the arrival rates in order to make the system less loaded.
We let the arrival rate $\lambda_i$ in each queue be as given in  Table \ref{low_arrival_rate_table}. We also assume that each queue has bounded capacity of $10$, and the user is dropped if the queue is full. The overall reward function among different agents is chosen as weighted-proportional fairness. The objective can be written as
\begin{align}
C(T) = &   \max_{\{a_{k,t}\}_{K\times T}}\  \sum_{k=1}^K w_k\log\left(\frac{1}{T}\sum_{t=1}^T a_{k,t}r_{k,t}\right) \label{eq:lat_prop_fair_eq}
\end{align}

The weights $w_k$ are given in Table \ref{weights_prop_fairness}. The weights were generated from a Normal distribution with mean $1$, and variance $0.1$. After sampling, the weights were normalized to make $\sum_k w_k = 1$.

\begin{table}[!htbp]
	\begin{center}
\begin{tabular}{|c |c |c | c| c| c| c| c|} 
 \hline
 $\lambda_1$ & $\lambda_2$ & $\lambda_3$ & $\lambda_4$ & $\lambda_5$ & $\lambda_6$ & $\lambda_7$ &$\lambda_8$\\ [0.5ex] 
 \hline
 0.014& 0.028& 0.042& 0.056& 0.069& 0.083& 0.097& 0.11\\ [0.5ex] 
 \hline
\end{tabular}
\caption{Arrival rates $\lambda_{k}$ (in number of packets per step) for proportional fairness}
\label{low_arrival_rate_table}
\end{center}
\end{table}

\begin{table}[!htbp]
	\begin{center}
\begin{tabular}{|c |c |c | c| c| c| c| c|} 
 \hline
 $w_1$ & $w_2$ & $w_3$ & $w_4$ & $w_5$ & $w_6$ & $w_7$ &$w_8$\\ [0.5ex] 
 \hline
 0.146& 0.112& 0.145& 0.119& 0.119& 0.123& 0.114& 0.122\\ [0.5ex] 
 \hline
\end{tabular}
\caption{Weights $w_{k}$ for weighted proportional fairness}
\label{weights_prop_fairness}
\end{center}
\end{table}

Again, we only evaluate  the model-free algorithm with $T=1000$. The gradient update equation for policy gradient algorithm as given in  Equation \eqref{eq:alpha_fair_grad} is used for $\alpha=1$.  Stochastic Gradient Ascent with learning rate $\eta=5\times10^{-1}$ is used to train the network. The value of discount factor $\gamma$ is set to $0.999$ and the batch size $N$ is kept as $80$.

We compare the proposed algorithm with the DQN Algorithm for Q-learning implementation. We use fairness at time $t$ or $C_2(t)$ as the reward for DQN network. The network consists of two fully connected hidden layers with $100$ units each,  ReLU activation function, and one output layer with linear activation. Adam optimizer with learning rate $0.01$ is used to optimize the network. The batch size is $64$ and the network is trained for $5000$ episodes.

The results for fairness maximization for this queuing system are provided in Figure \ref{fig:final_fig_weighted_prop_latency}. Similar to previous cases, we run each policy for $50$ times and median and we show the inter-quartile range for each policy. We note that compared to previous case, the objective increases as the episode progress. This is because the arrival rates are low, and queue lengths are short. For the entire episode, the proposed policy gradient algorithm outperforms the DQN policy. This is because of incorrect Q-learning (and DQN) cannot capture the non-linear functions of rewards, which is weighted proportional fairness in this case.

\begin{figure}[thbp]
\begin{center}
		\includegraphics[trim=0in .1in .4in .2in, clip,width=.48\textwidth]{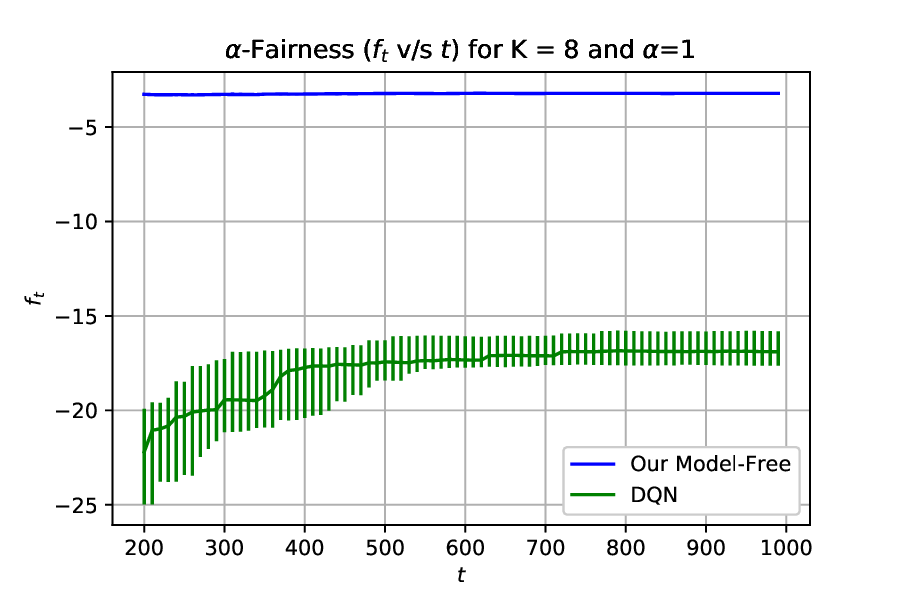}
	\caption{\small Weighted Proportional Fairness for the Queueing System, $f_t$ v/s $t$ (Best viewed in color)}
	\label{fig:final_fig_weighted_prop_latency}
	\end{center}
\end{figure}
\section{Conclusion} \label{conclusion}
This paper presents a novel average per step reward based formulation for optimizing joint objective function of long-term rewards of each objective for infinite horizon setup. In case of finite horizon, Markov policies may not be able to optimize the joint objective function, hence an average reward per step formulation is considered. A tabular model based algorithm which uses Dirichlet sampling to obtain regret bound of $\Tilde{O}\left(LKDS\sqrt{\frac{A}{T}}\right)$ for $K$ objectives scalarized using a $L$-Lipschitz concave function over a time horizon $T$ is provided where $S$ is the number of states and $D$ is the diameter of the underlying Markov Chain and $A$ is the number of actions available to the centralized controller. Further, a model free algorithm which can be efficiently implemented using neural networks is also proposed. The proposed algorithms outperform standard heuristics by a significant margin for maximizing  fairness in cellular scheduling problem, as well as for a multiple-queue queueing system.

Possible future works include modifying the framework to obtain actions from policies instead of probability values for infinite action space, and obtaining decentralized policies by introducing a message passing architecture. %
\clearpage
\bibliography{refs.bib,RL-ref.bib,refs_fairness.bib,Vaneet_cloud.bib,Vaneet_data.bib}  %
\newpage
\appendix
\section{Proof of Auxiliary Lemmas}
\subsection{Bounding the bias span for any MDP for any policy}
\begin{lemma}[Bounded Span of MDP]\label{lem:bounded_v_span}
For an MDP $\mathcal{M}$ with rewards $r^k(s,a)$ and transition probabilities $P$, for any stationary policy $\pi$ with average reward $\lambda_\pi^k$, the difference of bias of any two states $s$, and $s'$, is upper bounded by the diameter of the MDP $D$ as:
\begin{align}
    h_{\pi}^{P,k}(s) - h_{\pi}^{P,k}(s') \leq D~ \forall~s,s'\in \mathcal{S}.
\end{align}
\end{lemma}
\begin{proof}
Consider two states $s, s'\in \mathcal{S}$. Also, let $\tau$ be a random variable defined as:
\begin{align}
    \tau = \min\{t\geq 1: S_t = s', S_1 = s\}
\end{align}
Then, for any policy $\pi$, we have the following Bellman operator
\begin{align}
    h_{\pi}^{P,k}(s) &= r_{\pi}^k(s) - \lambda_\pi^k + <P_\pi(\cdot|s), h_{\pi}^{P,k}>\nonumber\\
    &= Th_{\pi}^{P,k}(s)(s)
\end{align}
where $P_\pi(s'|s) = \sum_{a}\pi(a|s)P(s'|s,a)$ and $r_\pi^k(s) = \sum_{a}\pi(a|s)r^k(s,a)$.

We also define another operator, 
\begin{align}
\bar{T}h(s)=
\begin{cases}
\min_{s,a}r^k(s,a) - \lambda_\pi^k + <P_\pi(\cdot|s), h>, &s\neq s'\\
h_{\pi}^{P,k}(s'), &s=s'
\end{cases}
\end{align}

Note that $(T - \bar{T})h_{\pi}^{P,k}(s) = r_\pi^k(s,a) - \min_{s,a}r^k(s,a) \ge 0$, for all $s$. Hence, we have $\bar{T}h(s) \le Th(s) = h_{\pi}^{P,k}(s)$, for all $s$. Further, for any two vectors $u, v$, where all the elements of $u$ are not smaller than $w$ we have $\bar{T}u \ge \bar{T}w$. Hence, we have $\bar{T}^nh_{\pi}^{P,k}(s) \le h_{\pi}^{P,k}(s)$ for all $s$. Unrolling the recurrence, we have
\begin{align}
    h_{\pi}^{P,k}(s) \ge \bar{T}^nh_{\pi}^{P,k}(s) = \mathbb{E}\left[-(\lambda_\pi^k - \min_{s,a}r^k(s,a))(n\wedge\tau) + h_{\pi}^{P,k}(S_{n\wedge\tau})\right]
\end{align}
For $\lim n\to \infty$, we have $h_{\pi}^{P,k}(s) \ge h_{\pi}^{P,k}(s') - D$, completing the proof.
\end{proof}

\subsection{Bounding the bias span for MDP with transition probabilities $P_e^k$}
\begin{lemma}[Bounded Span of optimal MDP]\label{lem:bounded_v_span_of_optimal_MDP}
For a MDP with rewards $r^k(s,a)$ and transition probabilities $P_e^k\in\mathcal{P}_{t_e}$, for policy $\pi_e$, the difference of bias of any two states $s$, and $s'$, is upper bounded by the diameter of the true MDP $D$ as:
\begin{align}
    h_{\pi_e}^{P_e^k,k}(s) - h_{\pi_e}^{P_e^k,k}(s') \leq D~ \forall~s,s'\in \mathcal{S}.
\end{align}
\end{lemma}
\begin{proof}
Note that $\lambda_{\pi_e}^{P_e^k,k} \ge \lambda_{\pi_e}^{P',k}$ for all $P'\in\mathcal{P}_{t_e}$. Now, consider the following Bellman equation:
\begin{align}
    h_{\pi_e}^{P_e^k,k}(s) &= r_{\pi_e}^k(s,a) - \lambda_{\pi_e}^{P_e^k,k} + <P_{\pi_e,e}^k(\cdot|s), h_{\pi_e}^{P_e^k,k}>\nonumber\\
    &= Th_{\pi_e}^{P_e^k,k}(s)
\end{align}
where $r_{\pi_e}(s) = \sum_{a}\pi_e(a|s)r^k(s,a)$ and $P_{\pi_e,e}^k(s'|s) = \sum_{a}\pi(a|s)P_e^k(s'|s,a)$.

Consider two states $s, s'\in \mathcal{S}$. Also, let $\tau$ be a random variable defined as:
\begin{align}
    \tau = \min\{t\geq 1: S_t = s', S_1 = s\}
\end{align}

We also define another operator, 
\begin{align}
\bar{T}h(s)=
\begin{cases}
\min_{s,a}r^k(s,a) - \lambda_{\pi_e}^{P_e^k,k} + <P_{\pi_e}(\cdot|s), h>, &s\neq s'\\
h_{\pi_e}^{P_e^k,k}(s'), &s=s'
\end{cases}
\end{align}
where $P_{\pi_e}(\cdot|s) = \sum_{a}\pi_e(a|s)P(s'|s,a)$.

Now, for any $s\in\mathcal{S}$, note that
\begin{align}
    h(s) &= Th(s)\\
    &=\max_{P'\in\mathcal{P}_{t_e}}\left(r_{\pi_e}^k(s) -\lambda_{\pi_e}^{P_e^k,k} + <P_{\pi_e}', h>\right)\\
    &\ge r_{\pi_e}^k(s) -\lambda_{\pi_e}^{P_e^k,k} + <P_{\pi_e}, h>\\
    &\ge \min_{s,a}r^k(s,a) -\lambda_{\pi_e}^{P_e^k,k} + <P_{\pi_e}, h>\\
    &= \bar{T}h(s)
\end{align}
Further, for any two vectors $u, v$, where all the elements of $u$ are not smaller than $w$ we have $\bar{T}u \ge \bar{T}w$. Hence, we have $\bar{T}^nh_{\pi}^{P,k}(s) \le h_{\pi}^{P,k}(s)$ for all $s$. Unrolling the recurrence, we have
\begin{align}
    h_{\pi}^{P_e^k,k}(s) \ge \bar{T}^nh_{\pi}^{P_e^k,k}(s) = \mathbb{E}\left[-(\lambda_\pi^k - \min_{s,a}r^k(s,a))(n\wedge\tau) + h_{\pi}^{P_e^k,k}(S_{n\wedge\tau})\right]
\end{align}
For $\lim n\to \infty$, we have $h_{\pi}^{P_e^k,k}(s) \ge h_{\pi}^{P_e^k,k}(s') - D$, completing the proof.
\end{proof}

\newpage
\section{Proof of Lemmas from main text}
\subsection{Proof of Lemma \ref{lem:bound_average_by_bellman}}\label{app:average_error_is_bellman_error}
\begin{proof}
Note that for all $s\in\mathcal{S}$, we have:
\begin{align}
    V_{\gamma,\pi}^{\Tilde{P},k}(s) &= \mathbb{E}_{a\sim\pi}\left[Q_{\gamma,\pi}^{\Tilde{P},k}(s,a)\right]\\
    &= \mathbb{E}_{a\sim\pi}\left[B_{\gamma,\pi}^{\Tilde{P},k}(s,a) + r(s,a) + \gamma\sum_{s'\in\mathcal{S}}P(s'|s,a)V_{\gamma\pi}^{\Tilde{P},k}(s')\right]\label{eq:optimistic_MDP_lambda}
\end{align}
where Equation \eqref{eq:optimistic_MDP_lambda} follows from the definition of the Bellman error for state action pair $s,a$.

Similarly, for the true MDP, we have,
\begin{align}
    V_{\gamma,\pi}^{P,k}(s) &= \mathbb{E}_{a\sim\pi}\left[Q_{\gamma,\pi}^{P,k}(s,a)\right]\\
    &= \mathbb{E}_{a\sim\pi}\left[r(s,a)+ \gamma\sum_{s'\in\mathcal{S}}P(s'|s,a)V_{\gamma,\pi}^{P,k}(s')\right] \label{eq:true_MDP_lambda}
\end{align}

Subtracting Equation \eqref{eq:true_MDP_lambda} from Equation \eqref{eq:optimistic_MDP_lambda}, we get:
\begin{align}
V_{\gamma,\pi}^{\Tilde{P},k}(s) - V_{\gamma,\pi}^{P,k}(s) &= \mathbb{E}_{a\sim\pi}\left[B_{\gamma,\pi}^{ \Tilde{P},k}(s,a) + \gamma\sum_{s'\in\mathcal{S}}P(s'|s,a)\left(V_{\gamma,\pi}^{\Tilde{P},k} - V_{\gamma,\pi}^{\Tilde{P},k}\right)(s')\right]\\
&= \mathbb{E}_{a\sim\pi}\left[B_{\gamma,\pi}^{\Tilde{P},k}(s,a)\right] + \gamma\sum_{s'\in\mathcal{S}}P_{\pi}\left(V_{\gamma,\pi}^{\Tilde{P},k} - V_{\gamma,\pi}^{\Tilde{P},k}\right)(s')
\end{align}

Using the vector format for the value functions, we have,
\begin{align}
    \Bar{V}_\gamma^{\pi,\Tilde{P}} - \Bar{V}_\gamma^{\pi,P} &= \left(I-\gamma P_{\pi}\right)^{-1}\mathbb{E}_{a\sim\pi}\left[B_{\gamma,\pi}^{\Tilde{P},k}(s,a)\right]%
\end{align}
Now, converting the value function to average per-step reward we have,
\begin{align}
    \lambda_{\pi}^{\Tilde{P},k}\bm{1}_S - \lambda_{P,\pi}^k\bm{1}_S &= \lim_{\gamma\to1}(1-\gamma)\left(\Bar{V}_{\gamma,\pi}^{\Tilde{P},k} - \Bar{V}_{\gamma,\pi}^{P,k}\right)\\
    &= \lim_{\gamma\to1}(1-\gamma)\left(I-\gamma P_{\pi}\right)^{-1}\mathbb{E}_{a\sim\pi}\left[B_{\gamma,\pi}^{\Tilde{P},k}(s,a)\right]\label{eq:Bellman_error_limit}\\
    &= \left(\sum_{s,a}d_{\pi}(s,a) B_{\pi}^{\Tilde{P},k}(s,a)\right)\bm{1}_S
\end{align}
where the last equation follows from the definition of occupancy measures by \cite{puterman1994markov}, and 
the existence of the limit $\lim_{\gamma\to1}B_{\gamma,\pi}^{\Tilde{P},k}$ in Equation \eqref{eq:Bellman_error_limit} from Equation \eqref{eq:value_to_bias}.

\end{proof}
\subsection{Proof of Lemma \ref{lem:bound_bellman_s_a_main}}\label{app:bounding_bellman_error}
\begin{proof}
Starting with the definition of Bellman error in Equation \eqref{eq:Bellman_error_definition}, we get
\begin{align}
B_{\pi}^{\Tilde{P},k}(s,a) &= \lim_{\gamma\to1}\left(Q_{\gamma,\pi}^{\Tilde{P},k}(s,a) - \left(r(s,a) +\gamma \sum_{s'\in\mathcal{S}}P(s'|s,a)V_{\gamma,\pi}^{\Tilde{P},k}(s') \right)\right)\\
&=\lim_{\gamma\to1}\left(\left(r(s,a) + \gamma\sum_{s'\in\mathcal{S}} \Tilde{P}(s'|s,a)V_{\gamma,\pi}^{\Tilde{P},k}(s')\right) - \left(r(s,a) +\gamma \sum_{s'\in\mathcal{S}}P(s'|s,a)V_{\gamma,\pi}^{\Tilde{P},k}(s') \right)\right)\\
&= \lim_{\gamma\to1}\gamma\sum_{s'\in\mathcal{S}}\left( \Tilde{P}(s'|s,a) - P(s'|s,a)\right)V_{\gamma,\pi}^{\Tilde{P},k}(s')\label{eq:rewards_known}\\
&= \lim_{\gamma\to1}\gamma\left(\sum_{s'\in\mathcal{S}}\left( \Tilde{P}(s'|s,a) - P(s'|s,a)\right)V_{\gamma,\pi}^{\Tilde{P},k}(s') + V_{\gamma,\pi}^{\Tilde{P},k}(s) - V_{\gamma,\pi}^{\Tilde{P},k}(s)\right)\\
&= \lim_{\gamma\to1}\gamma\Bigg(\sum_{s'\in\mathcal{S}}\left( \Tilde{P}(s'|s,a) - P(s'|s,a)\right)V_{\gamma,\pi}^{\Tilde{P},k}(s') - \sum_{s'\in\mathcal{S}} \Tilde{P}(s'|s,a)V_{\gamma,\pi}^{\Tilde{P},k}(s)\nonumber\\
&~~~~~+ \sum_{s'\in\mathcal{S}} P(s'|s,a)V_{\gamma,\pi}^{\Tilde{P},k}(s)\Bigg)\\
&= \lim_{\gamma\to1}\gamma\left(\sum_{s'\in\mathcal{S}}\left( \Tilde{P}(s'|s,a) - P(s'|s,a)\right)\left(V_{\gamma,\pi}^{\Tilde{P},k}(s') - V_{\gamma,\pi}^{\Tilde{P},k}(s)\right)\right)\\
&= \left(\sum_{s'\in\mathcal{S}}\left( \Tilde{P}(s'|s,a) - P(s'|s,a)\right)\lim_{\gamma\to1}\gamma\left(V_{\gamma,\pi}^{\Tilde{P},k}(s') - V_{\gamma,\pi}^{\Tilde{P},k}(s)\right)\right)\label{eq:interchange_limit_and_expectation}\\
&= \left(\sum_{s'\in\mathcal{S}}\left( \Tilde{P}(s'|s,a) - P(s'|s,a)\right)h_{\pi}^{\Tilde{P},k}(s')\right)\label{eq:value_to_bias}\\
&\le \Big\|\Tilde{P}(\cdot|s,a) - P(\cdot|s,a)\Big\|_1\|h_{\pi}^{\Tilde{P},k}(\cdot)\|_\infty\label{eq:reward_holders}\\
&\le \epsilon_{s,a}\Tilde{D} \label{eq:bound_bias}
\end{align}
where Equation \eqref{eq:rewards_known} comes from the assumption that the rewards are known to the agent. Equation \eqref{eq:interchange_limit_and_expectation} follows from the fact that the difference between value function at two states is bounded. Equation \eqref{eq:value_to_bias} comes from the definition of bias term \cite{puterman1994markov} where $h$ is the bias of the policy $\pi$ when run on the sampled MDP. Equation \eqref{eq:reward_holders} follows from H\"{o}lder's inequality. In Equation \eqref{eq:bound_bias}, $\|h(\cdot)\|_\infty$ is bounded by the diameter $\Tilde{D}$ of the sampled MDP(from Lemma \ref{lem:bounded_v_span}). Also, the $\ell_1$ norm of probability vector difference is bounded from the definition.

Additionally, note that the $\ell_1$ norm in Equation \eqref{eq:reward_holders} is bounded by $2$. Thus the Bellman error is loose upper bounded by $2\Tilde{D}$ for all state-action pairs.
\end{proof}

\newpage

%
%
%
%
%
%
%
%
%
%
%
%
%
%
%
%
%
\if 0
We also want to bound the deviation of the estimates of the estimated transition probabilities of the Markov Decision Processes $\mathcal{M}$. For that we use $\ell_1$ deviation bounds from \cite{weissman2003inequalities}. Consider, the following event,
\begin{align}
    \mathcal{E}_t = \left\{\|\hat{P}(\cdot|s, a) - P(\cdot|s, a)\|_1 \leq \sqrt{\frac{14S\log(2AT)}{\max\{1, n(s,a)\}}}\forall (s,a)\in\mathcal{S}\times\mathcal{A}\right\}
\end{align}
where $n=\sum_{t'}^t {\bf 1}_{\{s_{t'} = s, a_{t'}= a\}}$. Then we have, the following lemma:

\begin{lemma}
The probability that the event $\mathcal{E}_t$ fails to occur us upper bounded by $\frac{1}{20t^6}$.
\end{lemma}
\fi
\subsection{Proof of Lemma \ref{lem:deviation_of_probability_estimates}}
\begin{proof}
From the result of \cite{weissman2003inequalities}, the $\ell_1$ distance of a probability distribution over $S$ events with $n$ samples is bounded as:
\begin{align}
    \mathbb{P}\left(\|P(\cdot|s,a) - \hat{P}(\cdot|s,a)\|_1\geq \epsilon\right)&\leq (2^S-2)\exp{\left(-\frac{n(s,a)\epsilon^2}{2}\right)}\nonumber\\
    &\le (2^S)\exp{\left(-\frac{n(s,a)\epsilon^2}{2}\right)}
\end{align}

Thus, for $\epsilon = \sqrt{\frac{2}{n(s,a)}\log(2^S20 SAt^7)}\leq \sqrt{\frac{14S}{n(s,a)}\log(2At)} \leq \sqrt{\frac{14S}{n(s,a)}\log(2AT)}$, we have
\begin{align}
    \mathbb{P}\left(\|P(\cdot|s,a) - \hat{P}(\cdot|s,a)\|_1\geq \sqrt{\frac{14S}{n(s,a)}\log(2At)}\right)&\leq (2^S)\exp{\left(-\frac{n(s,a)}{2}\frac{2}{n(s,a)}\log(2^S20 SAt^7)\right)}\\
    &= 2^S \frac{1}{2^S 20 SAt^7}\\
    &= \frac{1}{20ASt^7}
\end{align}

We sum over the all the possible values of $n(s,a)$ till $t$ time-step to bound the probability that the event $\mathcal{E}_t$ does not occur as:
\begin{align}
    \sum_{n(s,a)=1}^t \frac{1}{20SAt^7} \leq \frac{1}{20SAt^6}
\end{align}

Finally, summing over all the $s,a$, we get
\begin{align}
    \mathbb{P}\left(\|P(\cdot|s,a) -\hat{P}(\cdot|s,a) \|_1\geq \sqrt{\frac{14S}{n(s,a)}\log(2At)}~\forall s,a\right) \leq \frac{1}{20t^6}
\end{align}
\end{proof}

\end{document}